\documentclass{article}

\usepackage{natbib}

\usepackage{bm}
\usepackage{PRIMEarxiv}
\usepackage{amsthm}
\usepackage{amssymb}
\usepackage{graphicx,amsmath,amsfonts,amssymb,bm,hyperref,url,breakurl,epsfig,epsf,color,fullpage,MnSymbol,mathbbol, fmtcount, semtrans
} 
\numberwithin{equation}{section}

\usepackage{titlesec}
\usepackage[]{mdframed}
\usepackage[ruled,algo2e]{algorithm2e}
\usepackage{enumitem}
\usepackage{booktabs}

\usepackage{tikz}
\usepackage{pgfplots}

\usetikzlibrary{pgfplots.groupplots}
\usepackage{amsmath,amssymb}

\titleformat{\paragraph}
{\normalfont\normalsize\bfseries}{\theparagraph}{1em}{}
\titlespacing*{\paragraph}
{0pt}{3.25ex plus 1ex minus .2ex}{1.5ex plus .2ex}

\usepackage{caption,subcaption}

\usepackage[bottom,hang,flushmargin]{footmisc} 
\usepackage[font=small]{caption}

\usepackage{hyperref}
\definecolor{darkred}{RGB}{150,0,0}
\definecolor{darkgreen}{RGB}{0,150,0}
\definecolor{darkblue}{RGB}{0,0,200}
\hypersetup{colorlinks=true, linkcolor=darkred, citecolor=darkgreen, urlcolor=black}

\newtheorem{theorem}{Theorem}
\newtheorem{lemma}{Lemma}

\newtheorem{proposition}{Proposition}
\newtheorem{definition}{Definition}
\newtheorem{assumption}{Assumption}

\newtheorem{remark}{Remark}

\newcommand{\R}{\mathbb{R}}

\usepackage[utf8]{inputenc} 
\usepackage[T1]{fontenc}    
\usepackage{hyperref}       
\usepackage{url}            
\usepackage{booktabs}       
\usepackage{amsfonts}       
\usepackage{nicefrac}       
\usepackage{microtype}      
\usepackage{lipsum}
\usepackage{fancyhdr}       
\usepackage{graphicx}       
\graphicspath{{media/}}     

\title{Classification with Abstention under Class-Conditional Error Constraints 
}

\author{
Mohammadreza M. Kalan\\
Univ Rennes, Ensai, CNRS,\\ CREST–UMR 9194,\\
F-35000 Rennes, France\\
\texttt{mohammadreza.kalan@ensai.fr}
\And
Yuyang Deng\\
Columbia University,\\ Department of Statistics\\
\texttt{yd2824@columbia.edu}
\And
Sanaz Hamidi\\
Univ Rennes, Ensai, CNRS,\\ CREST–UMR 9194,\\
F-35000 Rennes, France\\
\texttt{sanaz.hamidi@eleve.ensai.fr}
}

\begin{document}
\maketitle

\begin{abstract}
We study binary classification with abstention under separate class-conditional error constraints, with the objective of minimizing abstention while keeping both errors below prescribed thresholds. We characterize the distribution-free minimax rate of excess abstention risk, up to logarithmic factors, in terms of the complexity of the hypothesis class and the sample size. To make the framework amenable to computation with models such as neural networks, we introduce surrogate-loss formulations and derive finite-sample guarantees for excess surrogate ambiguity risk. We formulate the resulting learning task as a constrained optimization problem and characterize its computational complexity in the convex setting. Finally, we evaluate our approach on various datasets and compare its performance with a competing method for this problem.

\end{abstract}

\section{Introduction}

In many decision-making problems, misclassification can have serious consequences. In such applications, it is often preferable for a classifier to abstain when the available evidence is insufficient to make a prediction \citep{elyaniv2010foundations,mozannar2020consistent}. A medical diagnostic system, for example, may confidently identify some patients as healthy or at high risk, while referring uncertain cases to a specialist for further examination \citep{vazhentsev2026uncertainty,das2019interpretable,hamid2017machine,da2011diagnostic}. Similarly, in fraud detection, a system may automatically approve clearly legitimate transactions, block highly suspicious ones, and send ambiguous cases for manual review \citep{singh2026cost, alves2023fifar}. Abstention allows a classification system to avoid unreliable decisions, but abstaining too often limits its practical usefulness. The central challenge is therefore to minimize abstention while controlling classification errors.

In many such applications, the two types of classification errors have different consequences, making it important to control them separately. In medical diagnosis, for instance, missing a high-risk patient may have severe health consequences, while incorrectly flagging a healthy patient may lead to unnecessary tests or interventions. Separate control is also important when the classes are highly imbalanced or one class is scarce, since a small overall classification error does not necessarily imply a small error within each class. One common approach is to constrain the error associated with one class below a prescribed level while optimizing performance on the other, as in Neyman--Pearson classification \citep{scott2005neyman,rigollet2011neyman,tong2013plugin,tong2020neyman, kalan2025transfer}. In applications where both types of errors are critical, it is natural to instead require both class-conditional errors to remain below prescribed thresholds. However, because the two errors trade off against each other, these requirements may not be simultaneously achievable when a prediction is made for every observation. Abstention can then make it possible to satisfy both error constraints.

In this work, we study the problem of minimizing abstention while requiring the two class-conditional errors to remain below given levels $\alpha_0$ and $\alpha_1$. Given a finite labeled sample and a hypothesis class, the goal is to learn class-specific decision rules whose prediction regions capture sufficient probability mass under the corresponding class distributions to keep each class-conditional error below its prescribed threshold. These two regions may overlap, and their intersection forms the ambiguity region on which the classifier abstains. Thus, minimizing abstention amounts to minimizing how often observations fall in the ambiguity region, subject to keeping the class-conditional errors below the thresholds.

We first study the \textit{distribution-free statistical complexity} of this problem under the $0$-$1$ loss formulation and characterize its minimax rate up to logarithmic factors. For a hypothesis class with finite VC dimension, we establish finite-sample upper bounds based on empirical risk minimization and prove matching minimax lower bounds. We then introduce a surrogate-loss formulation that is more amenable to computation. For function classes with bounded Rademacher complexity, we derive finite-sample statistical guarantees for both the class-conditional constraints and the resulting abstention performance. The surrogate formulation leads to a constrained optimization problem, which we solve using a stochastic primal-dual  algorithm \citep{NIPS2012_c52f1bd6} and characterize its computational complexity in the convex setting. Finally, we evaluate the proposed approach on synthetic and real-world datasets and compare its performance with the plug-in method of \cite{lei2014classification}.

\section{Related Work}
\citet{lei2014classification} introduced a closely related framework, which considers classification under two class-conditional error constraints and minimizes the probability of the ambiguous region, on which the classifier abstains from making a  prediction. Under suitable regularity conditions, \citet{lei2014classification} characterizes the universally population-optimal classification regions (i.e., not restricted to a hypothesis class) through an extension of the Neyman--Pearson lemma \citep[Theorem 3.2.1]{lehmann1986testing}, and develops plug-in estimation procedures based on estimates of the conditional class probabilities, together with convergence guarantees under suitable regularity and margin conditions. In contrast, we study this constrained ambiguity-minimization problem in a distribution-free setting over a prescribed hypothesis class. In particular, we characterize the statistical minimax rate in terms of the complexity of the hypothesis class by establishing finite-sample upper bounds and matching lower bounds. We further develop a score-function formulation with surrogate losses and derive finite-sample statistical guarantees for function classes with bounded Rademacher complexity, making the framework directly applicable to modern learning models, including neural networks. We also provide a practical constrained optimization procedure for training such models. Finally, our empirical study directly compares the two approaches and demonstrates improved performance across a collection of real and synthetic datasets.

Other constrained formulations of classification with abstention, or rejection, have also been studied \citep{pietraszek2005optimizing,pietraszek2007use,li2006confidence,geifman2017selective,geifman2019selectivenet,denis2020consistency,shekhar2019binary,gangrade2021selective}. Among the formulations most closely related to ours, \citet{hanczar2008classification}, in a setting closely related to that of \citet{lei2014classification}, minimize the rejection rate subject to class-specific error constraints conditional on non-rejection. Their analysis, however, assumes a fixed classifier score and optimizes only the two rejection thresholds, without providing statistical guarantees that the resulting procedure converges to a population-optimal classifier. \citet{elyaniv2010foundations} study the risk--rejection trade-off in the noise-free realizable setting, deriving bounds on achievable trade-offs with particular emphasis on minimizing rejection under zero classification error among non-rejected predictions. Relatedly, the bounded-improvement formulation, which maximizes the probability of making a non-abstained prediction subject to a prescribed threshold on the total classification error among non-abstained predictions, and the bounded-abstention formulation, which minimizes this conditional classification error subject to a prescribed threshold on the probability of making a non-abstained prediction, were considered by \citet{pietraszek2005optimizing} and further studied by \citet{franc2023optimal}. \citet{franc2023optimal} characterize the corresponding population-optimal strategies in terms of the Bayes classifier and an appropriate, potentially randomized, abstention rule, without studying the sample complexity of these formulations.

Another widely studied formulation of classification with abstention is the
cost-based approach, in which abstention is assigned a fixed cost and the
objective is to minimize the expected cost of misclassification and abstention \citep{chow1970optimum, herbei2006classification, bartlett2008classification,yuan2010classification,ramaswamy2015consistent}. The classical formulation goes back to \citet{chow1970optimum}, who characterized
the Bayes-optimal rejection rule obtained by trading off the cost of
misclassification against the cost of rejection. \citet{herbei2006classification} studied statistical properties of plug-in and empirical risk minimization procedures under the classical cost-based reject-option formulation of \citet{chow1970optimum}.
\citet{bartlett2008classification} introduced a hinge-type convex surrogate loss for this formulation and established consistency and excess-risk guarantees. \citet{yuan2010classification} provided a
more general analysis of convex surrogate losses, giving conditions for
consistency and bounds relating excess surrogate risk to excess reject-option
risk. In contrast to our
formulation, these cost-based approaches control the trade-off between overall
classification error and abstention through a single penalized objective, rather
than imposing separate class-conditional error constraints; moreover, although some establish excess-risk convergence rates, they do not derive matching minimax upper and lower bounds. 

Finally, our formulation is also related to Neyman--Pearson classification, which addresses asymmetric class-conditional errors by controlling one error while minimizing the other \citep{scott2005neyman,rigollet2011neyman,tong2013plugin,zhao2016neyman,tong2018neyman,tong2020neyman,kalan2024distribution,kalan2024tight,kalan2025transfer,mousavi2026neyman}. Without abstention, the two class-conditional errors trade off against each other and therefore cannot in general be simultaneously controlled at arbitrarily small levels. Our formulation allows both errors to be controlled, with abstention serving as the price for this simultaneous control.
\section{Problem Setting}
Let $(\mathcal X,\Sigma)$ be a measurable feature space, and let $(X,Y)$ be a random pair taking values in $\mathcal X\times\{0,1\}$ with joint probability distribution $P$. We denote by $P_X$ the marginal distribution of $X$. For each $j\in\{0,1\}$, let $\pi_j:=\mathbb P(Y=j)$, and assume that $\pi_j>0$. Let $P_j$ denote the class-conditional distribution of $X$ given $Y=j$. 

Let $\mathcal H$ be a class of measurable functions
$h:\mathcal X\to\{0,1\}$. We choose a pair
$(h_0,h_1)\in\mathcal H\times\mathcal H$ and define
\[
    C_j(h_j):=\{x\in\mathcal X:h_j(x)=1\},
    \qquad j\in\{0,1\}.
\]
The sets $C_0(h_0)$ and $C_1(h_1)$ are viewed as candidate regions for
classifying points as class $0$ and class $1$, respectively. Points in $C_0(h_0)\setminus C_1(h_1)$ are classified as class $0$, and points in $C_1(h_1)\setminus C_0(h_0)$ are classified as class $1$. On the intersection $C_0(h_0)\cap C_1(h_1)$, the procedure abstains from classification; we call this intersection the ambiguous region. For $j\in\{0,1\}$, define the class-conditional error \[ R_j(h_j):=P_j(h_j(X)=0). \]
Thus, $R_j(h_j)$ is the probability that the class-$j$ candidate rule fails to include a point drawn from class $j$. For user-specified thresholds $\alpha_0,\alpha_1\in(0,1)$, the goal is
to choose $(h_0,h_1)\in\mathcal H\times\mathcal H$ such that the
class-conditional errors for classes $0$ and $1$ are at most
$\alpha_0$ and $\alpha_1$, while minimizing the probability
of the ambiguous region, denoted by
\[
    R_{\mathrm{amb}}(h_0,h_1)
    :=
    P_X\bigl(h_0(X)=1,\ h_1(X)=1\bigr).
\]This leads to the population problem
\begin{equation}\label{eq:population_problem}
\begin{aligned}
\underset{(h_0,h_1)\in\mathcal H\times\mathcal H}{\mathrm{minimize}}
    \quad & R_{\mathrm{amb}}(h_0,h_1) \\
\mathrm{subject\ to}
    \quad & R_0(h_0)\leq \alpha_0,
    \qquad
    R_1(h_1)\leq \alpha_1 .
\end{aligned}
\end{equation}
We assume throughout that the feasible set of~\eqref{eq:population_problem} is nonempty. We denote a solution to~\eqref{eq:population_problem} by
$(h_0^*,h_1^*)$ and its optimal value by $R_{\mathrm{amb}}^*$. The formulation in~\eqref{eq:population_problem} does not require $C_0(h_0)\cup C_1(h_1)=\mathcal X$. This does not create an additional undecided region: by convention, points in $\mathcal X\setminus(C_0(h_0)\cup C_1(h_1))$ may be assigned to class $0$. This gives the rule \[ h'_0(x) := \mathbf 1\{h_0(x)=1 \ \text{or}\ (h_0(x)=0,\ h_1(x)=0)\}, \] with $h_1$ kept unchanged. This modification does not change the ambiguity risk, and the class-conditional constraints remain satisfied. In proper learning, the final decision rule is required to belong to the hypothesis class $\mathcal H$; here, $h'_0$ need not belong to $\mathcal H$. We therefore allow improper learning.

Given $n$ i.i.d. samples $S^n=\{(X_i,Y_i)\}_{i=1}^{n}$, where $(X_i,Y_i)\sim P$, the goal is to learn $(\widehat h_0,\widehat h_1)\in\mathcal H\times\mathcal H$ whose
ambiguity risk is close to the optimal value $R_{\mathrm{amb}}^*$ while
satisfying the class-conditional constraints. We measure its statistical performance by the \textbf{excess ambiguity risk}
\begin{equation}\label{eq:excess_ambiguity_risk}
    \mathcal E_{\mathrm{amb}}(\widehat h_0,\widehat h_1)
    :=
    R_{\mathrm{amb}}(\widehat h_0,\widehat h_1)
    -
    R_{\mathrm{amb}}^* .
\end{equation}

\section{Minimax Rates} 
In this section, we characterize the statistical complexity of 
\eqref{eq:population_problem}.  We first prove an upper bound for an 
empirical constrained learner and then establish a matching minimax
lower bound.
\subsection{Empirical Constrained Learner and Upper Bound}
Let $n_j:=\sum_{i=1}^n \mathbf 1\{Y_i=j\}$ denote the number of  samples from class $j$, for $j\in\{0,1\}$. For $h\in\mathcal H$, define the empirical class-conditional errors \[ \widehat R_j(h) := \frac{1}{n_j}\sum_{i:Y_i=j}\mathbf 1\{h(X_i)=0\}, \qquad j\in\{0,1\}, \]
whenever $n_j>0$. Similarly, for a pair $(h_0,h_1)\in\mathcal H\times\mathcal H$, define the empirical ambiguity risk \[ \widehat R_{\mathrm{amb}}(h_0,h_1) := \frac{1}{n}\sum_{i=1}^n \mathbf 1\{h_0(X_i)=1,\ h_1(X_i)=1\}. \]
Let $\epsilon_0,\epsilon_1>0$ be slack parameters to be determined. The
empirical constrained learner is defined as any solution
$(\widehat h_0,\widehat h_1)\in\mathcal H\times\mathcal H$ to
\begin{equation}\label{eq:empirical_problem}
\begin{aligned}
\underset{(h_0,h_1)\in\mathcal H\times\mathcal H}{\mathrm{minimize}}
    \quad & \widehat R_{\mathrm{amb}}(h_0,h_1) \\
    \mathrm{subject\ to}
    \quad & \widehat R_0(h_0)\leq \alpha_0+\epsilon_0,
    \qquad
    \widehat R_1(h_1)\leq \alpha_1+\epsilon_1 .
\end{aligned}
\end{equation}

\begin{theorem}[Upper bound]
\label{thm:upper_bound}
Assume that $\mathcal H$ has finite VC dimension $d_{\mathcal H}$. Let
$\delta\in(0,1)$ and define the event $ \mathcal E_n:=\{n_0>0,\ n_1>0\}.$ On the event $\mathcal E_n$, for $j\in\{0,1\}$, set
$\epsilon_j(n_j,\delta):=C\sqrt{\{d_{\mathcal H}\log(e n_j)+\log(1/\delta)\}/n_j}$, where
$C>0$ is a universal constant. Let
$(\widehat h_0,\widehat h_1)$ be any solution to
\eqref{eq:empirical_problem} with slack parameters
$\epsilon_0/2$ and $\epsilon_1/2$. Then, conditional on $\mathcal E_n$,
with probability at least $1-2\delta-\frac{\delta}{\mathbb{P}(\mathcal{E}_n)}$,
\[
    R_j(\widehat h_j)\leq \alpha_j+\epsilon_j(n_j,\delta),\ \text{for}\ j\in\{0,1\}, \quad \text{and} \quad
    \mathcal E_{\mathrm{amb}}(\widehat h_0,\widehat h_1)
    \leq
    C
    \sqrt{
    \frac{
    d_{\mathcal H}\log(e n)+\log(1/\delta)
    }{
    n
    }} .
\]
for a universal constant $C>0$.
\end{theorem}
The preceding upper bound relies on positive slack in the class-conditional
constraints. This relaxation is unavoidable for deterministic learners:
even with complete knowledge of the underlying distribution, imposing the
strict constraints \(R_0(h_0)<\alpha_0\) and \(R_1(h_1)<\alpha_1\) may force
every feasible pair to have maximal ambiguity excess risk. The following
proposition formalizes this impossibility.

\begin{proposition}
\label{prop:strict_feasibility_impossibility}
Fix $\alpha_0,\alpha_1\in(0,1/2)$. There exist a domain $\mathcal X$, a hypothesis class $\mathcal H$ with $d_{\mathcal H}>0$, and a distribution $P$ with $\pi_0=\pi_1=1/2$ such that every pair $(h_0,h_1)\in\mathcal H\times\mathcal H$ satisfying
$
R_0(h_0)<\alpha_0$ and 
$R_1(h_1)<\alpha_1$
has ambiguity excess risk $
\mathcal E_{\mathrm{amb}}(h_0,h_1)=1$.
\end{proposition}
We next define a randomized learning rule that satisfies the class-conditional constraints without slack. Let $(\widehat h_0,\widehat h_1)$ be the empirical hypotheses obtained by the procedure in \eqref{eq:empirical_problem}. Let $\mathbf{1}$ denote the classifier that is identically equal to one on $\mathcal X$. On the event $\mathcal E_n=\{n_0>0,n_1>0\}$, define $q_j
:=
\frac{\alpha_j}{\alpha_j+\epsilon_j(n_j,\delta)}, \ j\in\{0,1\}.$ Given the data $S^n$, independently draw $B_j\sim \operatorname{Bernoulli}(q_j)$, and define the randomized classifier
\begin{align}\label{random_learner}
\widetilde h_j
:=
B_j\widehat h_j+(1-B_j)\mathbf{1},
\qquad j\in\{0,1\}.
\end{align}
Let $\mathbb E_{\zeta}$ denote expectation with respect to the auxiliary random variables $\zeta=(B_0,B_1)$. Then, define
\[
\overline R_j(\widetilde h_j):=\mathbb E_{\zeta}\left[R_j(\widetilde h_j)\right],\ \  \text{for }\ j\in\{0,1\},
\qquad \text{and} \qquad
\overline R_{\mathrm{amb}}(\widetilde h_0,\widetilde h_1):=\mathbb E_{\zeta}\left[R_{\mathrm{amb}}(\widetilde h_0,\widetilde h_1)\right].
\]
\begin{theorem}[Upper bound for the randomized learner]
\label{thm:randomized_upper_bound}
Assume that $\mathcal H$ has finite VC dimension $d_{\mathcal H}$, and let $\delta\in(0,1)$. On the event
$\mathcal E_n$, let $\epsilon_j(n_j,\delta)$ be as in
Theorem~\ref{thm:upper_bound}. Let $(\widehat h_0,\widehat h_1)$ be any
solution of \eqref{eq:empirical_problem} with slack parameters
$\epsilon_0/2$ and $\epsilon_1/2$, and let
$(\widetilde h_0,\widetilde h_1)$ be the randomized learner defined in
\eqref{random_learner}. Then, conditional on $\mathcal E_n$, with
probability at least $1
-
2\delta
-
\frac{\delta}{\mathbb P(\mathcal E_n)}
-
\frac{
2\exp\left(-c n\min\{\pi_0,\pi_1\}\right)
}{
\mathbb P(\mathcal E_n)
}$ with respect to the training sample $S^n$, we have
\[
\overline R_j(\widetilde h_j)\leq \alpha_j,\quad j\in\{0,1\},
\qquad\text{and}\qquad
\overline{\mathcal E}_{\mathrm{amb}}(\widetilde h_0,\widetilde h_1)
\leq
C_{\alpha,\pi}
\sqrt{
\frac{
d_{\mathcal H}\log(en)+\log(1/\delta)
}{
n
}}.
\]
Here, $
\overline{\mathcal E}_{\mathrm{amb}}(\widetilde h_0,\widetilde h_1)
:=
\overline R_{\mathrm{amb}}(\widetilde h_0,\widetilde h_1)
-
R_{\mathrm{amb}}^*$, and \(C_{\alpha,\pi}>0\) is a constant depending only on
\(\alpha_0,\alpha_1,\pi_0\), and \(\pi_1\).
\end{theorem}

\subsection{Minimax Lower Bound}

To derive a matching minimax lower bound for the problem \eqref{eq:population_problem}, we first specify the class of distributions over which the minimax problem is formulated.

\begin{definition}[Class of distributions]\label{def:class_distributions}
Fix a hypothesis class $\mathcal H$ with finite VC dimension $d_{\mathcal H}$. For $\alpha_0,\alpha_1\in(0,1)$, let $
\mathcal P_{\mathcal H}(\alpha_0,\alpha_1,\pi_0,\pi_1)$ denote the class of distributions $P$ on $\mathcal X\times\{0,1\}$ such that $
P(Y=0)=\pi_0, P(Y=1)=\pi_1
$, and for which the problem \eqref{eq:population_problem} is feasible; namely, there exist $h_0,h_1\in\mathcal H$ such that
\[
R_0(h_0)\le \alpha_0,
\qquad
R_1(h_1)\le \alpha_1.
\]
\end{definition}

We next specify the class of learning rules considered in the minimax formulation. In line with the randomized learner used in the upper bound, we allow the learning rule to use internal randomness. The constraints are required to hold in expectation with respect to this internal randomness, with high probability over the training sample.


\begin{definition}[Approximate randomized $(\alpha_0,\alpha_1)$-learner]\label{def:learners}
Let $r_0,r_1\ge 0$ and $0<\delta<1$. Let
$\mathcal H^+=\mathcal H\cup\{\mathbf 1\}$, where $\mathbf 1$ denotes the
constant-one classifier. For every $n\ge 1$, we call a randomized learning rule
\[
(S^n,\zeta)\longmapsto
(\widetilde h_0,\widetilde h_1)\in \mathcal H^+\times \mathcal H^+
\]
an \textbf{$(r_0,r_1,\delta)$-approximate randomized
$(\alpha_0,\alpha_1)$-learner} if the auxiliary randomness $\zeta$ is independent of $S^n$, and for every distribution
$P\in\mathcal P_{\mathcal H}(\alpha_0,\alpha_1,\pi_0,\pi_1)$,
\[
\mathbb P_{S^n\sim P^n}
\left(
\overline R_0(\widetilde h_0)\le \alpha_0+r_0,\;
\overline R_1(\widetilde h_1)\le \alpha_1+r_1
\right)
\ge 1-\delta .
\]
\end{definition}

\begin{figure}[t]
    \centering
    \includegraphics[height=0.3\textheight]{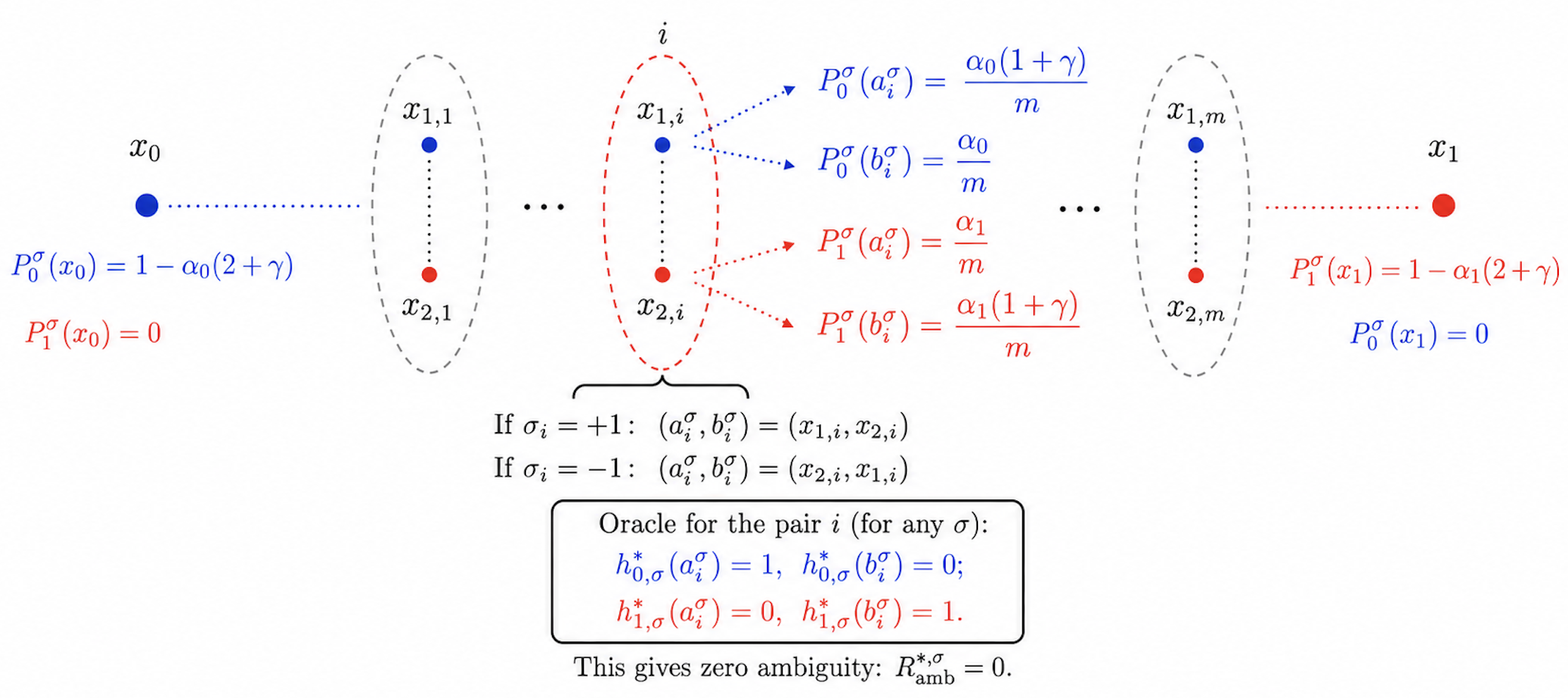}
    \caption{Illustration of the lower bound construction in Theorem \ref{thm:minimax_lower_bound_randomized}. Here, $m=\left\lfloor \frac{d_{\mathcal H}-2}{2} \right\rfloor \sim d_{\mathcal H}$. The unknown sign vector $\sigma\in \{-1,+1\}^m$ determines, in each pair $(x_{1,i},x_{2,i})$, which point is $a_i^{\sigma}$ and which is $b_i^{\sigma}$. The oracle accepts $a_i^{\sigma}$ for class $0$  and $b_i^{\sigma}$ for class $1$, satisfying $R_0=\alpha_0$ and $R_1=\alpha_1$ with zero ambiguity. Since the mass gap is only of order $\gamma \sim \sqrt{d_{\mathcal{H}}/n}$, the different values of $\sigma$ are difficult to distinguish; errors in recovering $\sigma$, together with the feasibility constraints, force overlap between the two regions, yielding an $\Omega\left(\sqrt{d_{\mathcal H}/n}\right)$ excess ambiguity lower bound.}
    \label{fig:minimax_construction}
\end{figure}
\begin{theorem}
\label{thm:minimax_lower_bound_randomized}
Fix a hypothesis class $\mathcal H$ with finite VC dimension
$d_{\mathcal H}\ge 4$, and suppose that the constant-one classifier
$\mathbf 1$ belongs to $\mathcal H$. Let
$\alpha_0,\alpha_1\in(0,\tfrac{49}{100})$ and
$0<\delta<1/6$. Suppose that \(n\) is large enough so that $
n\ge c\, d_{\mathcal H}$ for a sufficiently large constant \(c>0\). Let \(r_0,r_1\ge 0\) satisfy $
    \frac{r_0}{\alpha_0}
    +
    \frac{r_1}{\alpha_1}
    \le
    c_0\sqrt{\frac{d_{\mathcal H}}{n}}$
for a sufficiently small constant \(c_0>0\). Then, for every randomized
$(r_0,r_1,\delta)$-approximate
$(\alpha_0,\alpha_1)$-learner
$
(S^n,\zeta)
\longmapsto
(\widetilde h_0,\widetilde h_1)
\in
\mathcal H\times\mathcal H$, there exists a distribution
$P\in\mathcal P_{\mathcal H}(\alpha_0,\alpha_1,\tfrac12,\tfrac12)$
such that
\[
\mathbb P_{S^n\sim P^n}
\left(
\overline{\mathcal E}_{\mathrm{amb}}
(\widetilde h_0,\widetilde h_1)
\ge
c_1\sqrt{\frac{d_{\mathcal H}}{n}}
\right)
\ge
c_2,
\]
where \(c_1,c_2>0\) are universal constants depending only on \(\alpha_0\) and \(\alpha_1\).
\end{theorem}

\begin{remark}
Theorem \ref{thm:minimax_lower_bound_randomized} applies to the class of randomized
\((r_0,r_1,\delta)\)-approximate learners whenever $
    \frac{r_0}{\alpha_0}+\frac{r_1}{\alpha_1}
    \le
    c_0\sqrt{\frac{d_{\mathcal H}}{n}}$ for a sufficiently small constant $c_0>0$. This class contains the exact-feasible case \(r_0=r_1=0\). It also contains
deterministic learners as the special case in which the auxiliary randomness is
degenerate. Thus the lower bound is robust: even after allowing both
randomization and constraint slack of the same order as the target statistical rate, the ambiguity
excess risk remains at least of order $
    \sqrt{\frac{d_{\mathcal H}}{n}}$. 
\end{remark}

The intuition behind the proof of the lower bound is illustrated in Figure~\ref{fig:minimax_construction}. The construction consists of a collection of paired points, with the number of pairs proportional to the VC dimension $d_{\mathcal H}$. For each pair, an unknown sign determines a small perturbation of the class-conditional probabilities, arranged in opposite directions for the two classes. If these signs were known, an oracle could select complementary points for the two class-specific regions, satisfy both error constraints, and achieve zero ambiguity. However, the perturbations are chosen small enough that, from $n$ observations, the learner cannot reliably identify the correct choice across all pairs. Under the class-conditional feasibility requirements, such errors necessarily translate into overlap between the two prediction regions, and hence into nonzero ambiguity. Choosing the perturbation at the statistical indistinguishability scale then yields an excess ambiguity of order $\sqrt{d_{\mathcal H}/n}$.


    

\section{Surrogate Formulation and Statistical Guarantees}
The formulation developed so far is based on binary-valued classifiers and the $0$--$1$ loss. Although this formulation provides a natural statistical description of the ambiguity minimization problem, directly optimizing the associated empirical objective is generally computationally difficult. In particular, the indicator losses appearing in both the class-conditional constraints and the ambiguity objective are discontinuous and typically lead to nonconvex combinatorial optimization problems. We therefore replace the binary hypotheses by real-valued score functions and the indicator losses by surrogate losses. In addition to yielding more tractable optimization problems, surrogate losses incorporate the distance of each observation from the corresponding decision boundary, thereby encouraging classifiers with larger margins and improved robustness.

Let $\mathcal F$ be a class of measurable functions
$f:\mathcal X\to\mathbb R$. Each $f\in\mathcal F$ induces the binary
classifier $
    h_f(x):=\mathbf 1\{f(x)\geq 0\}.$ Accordingly, for a pair $(f_0,f_1)\in\mathcal F\times\mathcal F$, the point
$x\in\mathcal X$ is declared ambiguous whenever $
    f_0(x)\geq 0$ and 
    $f_1(x)\geq 0.$ Thus, the induced $0$--$1$ class-conditional and ambiguity losses are defined by $\ell_j^{0\text{-}1}(f_j;x):=\mathbf 1\{f_j(x)<0\}$, $j\in\{0,1\}$, and $\ell_{\mathrm{amb}}^{0\text{-}1}(f_0,f_1;x):=\mathbf 1\{f_0(x)\geq 0,\ f_1(x)\geq 0\}$.
    \begin{definition}[Surrogate loss]\label{def:surrogate_loss}
Assume that there exists $B>0$ such that
$\sup_{f\in\mathcal F}\sup_{x\in\mathcal X}|f(x)|\le B$.
A function $\phi:\mathbb R\to \mathbb R_{+}$ is called an
$L_\phi$-Lipschitz surrogate loss over $\mathcal F$ if it is
nondecreasing, satisfies $\phi(0)=1$, $
|\phi(u)-\phi(v)|\le L_\phi|u-v|$ for all $u,v\in[-B,B]$, and there exists a constant $C_\phi>1$ such that $
\sup_{|u|\le B}\phi(u)\le C_\phi.$
\end{definition}
If $\mathcal F$ is uniformly bounded on $\mathcal X$, then standard surrogate losses such as the hinge, logistic, and exponential losses satisfy all the conditions in Definition~\ref{def:surrogate_loss}.
\begin{definition}[Surrogate risks]
\label{def:surrogate_risks}
Let $(u)_+:=\max\{u,0\}$. For $(f_0,f_1)\in\mathcal F\times\mathcal F$, define $R_{\phi,j}(f_j):=\mathbb E_{P_j}[\phi(-f_j(X))]$, $j\in\{0,1\}$, and the pointwise surrogate ambiguity losses by
\[
\ell_{\mathrm{amb},\phi}^{\mathrm{add}}(f_0,f_1;x)
:=
\left[\phi(f_0(x))+\phi(f_1(x))-1\right]_+,
\qquad
\ell_{\mathrm{amb},\phi}^{\mathrm{prod}}(f_0,f_1;x)
:=
\phi(f_0(x))\phi(f_1(x)).
\]
For $\diamond\in\{\mathrm{add},\mathrm{prod}\}$, define $
R_{\mathrm{amb},\phi}^{\diamond}(f_0,f_1)
:=
\mathbb E_{P_X}\!\left[
\ell_{\mathrm{amb},\phi}^{\diamond}(f_0,f_1;X)
\right]$.
\end{definition}
Since $\phi$ is nondecreasing and $\phi(0)=1$, these surrogate risks upper-bound their $0$--$1$ counterparts: $R_j(h_{f_j})\leq R_{\phi,j}(f_j)$ for $j\in\{0,1\}$, and $R_{\mathrm{amb}}(h_{f_0},h_{f_1})\leq R_{\mathrm{amb},\phi}^{\diamond}(f_0,f_1)$ for $\diamond\in\{\mathrm{add},\mathrm{prod}\}$.

The additive formulation is convex whenever $\phi$ is convex and
$\mathcal F$ admits a convex parameterization. However, it may assign a
large loss even when only one classifier accepts the point, since a
sufficiently large positive score can outweigh a negative score from the
other classifier. In contrast, ambiguity occurs only when both classifiers
accept the same point, and the multiplicative formulation assigns a smaller
penalty whenever at least one score is negative. It is therefore better
aligned with the ambiguity event, although it is generally nonconvex.

For $\diamond\in\{\mathrm{add},\mathrm{prod}\}$, consider the population
optimization problem
\begin{align}\label{surrogate-population}
\underset{(f_0,f_1)\in\mathcal F\times\mathcal F}{\operatorname{minimize}}
&\quad
R_{\mathrm{amb},\phi}^{\diamond}(f_0,f_1)\nonumber\\
\operatorname{subject\ to}
&\quad
R_{\phi,0}(f_0)\leq \alpha_0,
\qquad
R_{\phi,1}(f_1)\leq \alpha_1.
\end{align}
We denote a solution to \eqref{surrogate-population} by $(f_0^{\diamond,*},f_1^{\diamond,*})$ and its optimal value by $R_{\mathrm{amb},\phi}^{\diamond,*}$.
The corresponding surrogate excess ambiguity risk is
\[
\mathcal E_{\mathrm{amb},\phi}^{\diamond}(f_0,f_1)
:=
R_{\mathrm{amb},\phi}^{\diamond}(f_0,f_1)
-
R_{\mathrm{amb},\phi}^{\diamond,*}.
\]
We next introduce the empirical counterpart of
\eqref{surrogate-population}, obtained by replacing the
population surrogate risks with their empirical analogues and allowing
small slacks in the class-conditional constraints. Recall that $n_j$ denotes
the number of samples from class $j$. For $n_j>0$, define
\[
\widehat R_{\phi,j}(f_j)
:=
\frac{1}{n_j}
\sum_{i:Y_i=j}
\phi\!\left(-f_j(X_i)\right),
\qquad j\in\{0,1\}.
\]
For $\diamond\in\{\mathrm{add},\mathrm{prod}\}$, define the empirical
surrogate ambiguity risk by
\[
\widehat R_{\mathrm{amb},\phi}^{\diamond}(f_0,f_1)
:=
\frac{1}{n}\sum_{i=1}^{n}
\ell_{\mathrm{amb},\phi}^{\diamond}(f_0,f_1;X_i).
\]
Let $\epsilon_0,\epsilon_1\geq0$ be slack parameters. For
$\diamond\in\{\mathrm{add},\mathrm{prod}\}$, the empirical surrogate
learner is defined as any solution
$(\widehat f_0,\widehat f_1)\in\mathcal F\times\mathcal F$
to
\begin{align}\label{empirical_surrogate}
\underset{(f_0,f_1)\in\mathcal F\times\mathcal F}{\operatorname{minimize}}
&\quad
\widehat R_{\mathrm{amb},\phi}^{\diamond}(f_0,f_1)\nonumber\\
\operatorname{subject\ to}
&\quad
\widehat R_{\phi,0}(f_0)\leq \alpha_0+\epsilon_0,
\qquad
\widehat R_{\phi,1}(f_1)\leq \alpha_1+\epsilon_1.
\end{align}
To analyze the statistical performance of the empirical solutions, we measure
the complexity of the score class $\mathcal F$ through its Rademacher
complexity.

\begin{definition}[Rademacher complexity]
\label{def:rademacher_complexity}
Let $X_1,\ldots,X_m$ be i.i.d.\ samples drawn from a distribution $Q$ on
$\mathcal X$. Define the empirical Rademacher complexity by
\[
\widehat{\mathfrak R}_m(\mathcal F)
:=
\mathbb E_{\sigma}\!\left[
\sup_{f\in\mathcal F}
\left|
\frac{1}{m}\sum_{i=1}^m \sigma_i f(X_i)
\right|
\right],
\]
where $\sigma_1,\ldots,\sigma_m$ are independent uniform
$\{-1,1\}$-valued random variables. The Rademacher complexity of
$\mathcal F$ with respect to $Q$ is defined by $
\mathfrak R_m^{Q}(\mathcal F)
:=
\mathbb E\!\left[
\widehat{\mathfrak R}_m(\mathcal F)
\right]$, where the expectation is taken with respect to the i.i.d.\ sample.
\end{definition}

\begin{assumption}
\label{ass:rademacher_complexity}
There exists a constant $B_{\mathcal F}>0$ such that $
\mathfrak R_m^{Q}(\mathcal F)
\leq
\frac{B_{\mathcal F}}{\sqrt m}$
for all distributions $Q$ on $\mathcal X$.
\end{assumption}
\begin{remark}
\label{rem:rademacher_examples}
Assumption~\ref{ass:rademacher_complexity} is satisfied by many commonly
used function classes under standard boundedness conditions. For instance,
linear predictors with bounded coefficients and neural networks with bounded
weights have Rademacher complexity of order $m^{-1/2}$ when the input domain
is bounded.
\end{remark}

For $\diamond\in\{\mathrm{add},\mathrm{prod}\}$, define
$\kappa_{\mathrm{add}}:=1$, $\kappa_{\mathrm{prod}}:=C_\phi$,
$M_{\mathrm{add}}:=2C_\phi$, and $M_{\mathrm{prod}}:=C_\phi^2$.
For $m\geq1$ and $\delta\in(0,1)$, let
\[
\epsilon_{\phi}(m,\delta)
:=
\frac{4L_\phi B_{\mathcal F}}{\sqrt m}
+
C_\phi\sqrt{\frac{\log(6/\delta)}{2m}},
\qquad
\epsilon_{\mathrm{amb},\phi}^{\diamond}(n,\delta)
:=
C\left(
\frac{\kappa_\diamond L_\phi B_{\mathcal F}}{\sqrt n}
+
M_\diamond\sqrt{\frac{\log(6/\delta)}{n}}
\right),
\]
where $C>0$ is a universal constant.
\begin{theorem}
\label{thm:empirical_surrogate_upper_bound}
Suppose that Definition~\ref{def:surrogate_loss} and
Assumption~\ref{ass:rademacher_complexity} hold. Fix
$\diamond\in\{\mathrm{add},\mathrm{prod}\}$. Let $
\mathcal E_n:=\{n_0>0,\ n_1>0\}$, and, on $\mathcal E_n$, let
$(\widehat f_0^\diamond,\widehat f_1^\diamond)$ be any solution of the
empirical problem \eqref{empirical_surrogate} with slack parameters $
\epsilon_j
=
\epsilon_{\phi}(n_j,\delta)$ for $j\in\{0,1\}$. Then, conditional on
$\mathcal E_n$, with probability at least $
1-\frac{2\delta}{3}-\frac{\delta}{3\mathbb P(\mathcal E_n)}$,
we have
\[
R_{\phi,j}(\widehat f_j^\diamond)
\leq
\alpha_j+2\epsilon_{\phi}(n_j,\delta),
\quad j\in\{0,1\},
\qquad\text{and}\qquad
\mathcal E_{\mathrm{amb},\phi}^{\diamond}
(\widehat f_0^\diamond,\widehat f_1^\diamond)
\leq
2\epsilon_{\mathrm{amb},\phi}^{\diamond}(n,\delta).
\]
Consequently, $
R_j(h_{\widehat f_j^\diamond})
\leq
\alpha_j
+
2\epsilon_{\phi}(n_j,\delta)$ for $j\in\{0,1\}$.
\end{theorem}

\section{Optimization and Computational Guarantees}
\label{sec:optimization}
In this section, we describe the optimization procedure used in our experiments to approximately solve the empirical problem \eqref{empirical_surrogate}. First we assume the two classifiers $f_0$ and $f_1$ are parameterized by $\theta_0$ and $\theta_1$, respectively, where $\theta_0, \theta_1 \in \Theta \subseteq \R^d$. For
$\diamond \in \{\mathrm{add},\mathrm{prod}\}$, The problem can then be rewritten as follows:
\begin{equation}
\begin{aligned}
\underset{(\theta_0,\theta_1)\in\Theta\times\Theta}{\operatorname{minimize}}
\quad &
\widehat R_{\mathrm{amb},\phi}^{\diamond}
\bigl(f_{\theta_0},f_{\theta_1}\bigr) \\
\text{subject to}
\quad &
\widehat R_{\phi,0}\bigl(f_{\theta_0}\bigr)
    \leq \alpha_0+\epsilon_0,\quad 
\widehat R_{\phi,1}\bigl(f_{\theta_1}\bigr)
    \leq \alpha_1+\epsilon_1 .
\end{aligned}
\label{empirical_surrogate_parametric}
\end{equation}
We further define $g(\theta_0,\theta_1):= \max\{ \widehat R_{\phi,0}\bigl(f_{\theta_0}\bigr) - \alpha_0 - \epsilon_0,  \widehat R_{\phi,1}\bigl(f_{\theta_1}\bigr)- \alpha_1 - \epsilon_1 \}$, so that the two constraints can be written compactly as
\(g(\theta_0,\theta_1)\leq 0\). To approximately solve this constrained problem, we employ the
one-projection algorithm of \cite{NIPS2012_c52f1bd6}. The algorithm first
applies gradient descent--ascent to the associated Lagrangian problem
\[
\min_{(\theta_0,\theta_1)\in\Theta\times\Theta}
\max_{\lambda\geq 0}
\left\{
\widehat R_{\mathrm{amb},\phi}^{\diamond}
\bigl(f_{\theta_0},f_{\theta_1}\bigr)
+\lambda g(\theta_0,\theta_1)
\right\}.
\]
After \(T\) iterations, it performs a single projection of the parameters
onto the feasible set. The complete procedure is presented in
Algorithm \ref{algorithm: SGD with One Projection}. Although our experiments use multilayer perceptrons (MLPs) with both
the additive and product ambiguity losses, the following guarantee
applies to the additive formulation when the objective and constraints
are convex in the parameters.

\begin{theorem}\label{thm:optimization}
Consider $\diamond=\mathrm{add}$ and write
$\theta=(\theta_0,\theta_1)$.
Let $B=\{\theta:\|\theta\|_2\leq1\}$ and
$K=\{\theta:g(\theta)\leq0\}$.
Assume that $K\subseteq B$, that there exists
$\theta\in\operatorname{int}(B)$ with $g(\theta)<0$,
and that $
\widehat{\theta}^\star=(\widehat{\theta}_0^\star,\widehat{\theta}_1^\star)
\in\operatorname*{arg\,min}_{\theta\in K}
\widehat R_{\mathrm{amb},\phi}^{\mathrm{add}}
(f_{\theta_0},f_{\theta_1})$ satisfies $\widehat{\theta}^\star\in\operatorname{int}(B)$.
Assume that $g$ and each mapping
$\theta\mapsto\ell_{\mathrm{amb},\phi}^{\mathrm{add}}
(f_{\theta_0},f_{\theta_1};X_i)$
are convex on a neighborhood of $B$.
For every $\theta\in B$ and $i\in\{1,\ldots,n\}$, suppose
\[
\sup_{v\in\partial_\theta\ell_{\mathrm{amb},\phi}^{\mathrm{add}}(f_{\theta_0},f_{\theta_1};X_i)}
\|v\|_2\leq G_1,\quad
\sup_{u\in\partial g(\theta)}\|u\|_2\leq G_2,\quad
|g(\theta)|\leq C_g,\quad
\inf_{\theta:\,g(\theta)=0}\inf_{u\in\partial g(\theta)}
\|u\|_2\geq\rho>0.
\]
For $\delta\in(0,1)$, define
$
c_\delta:=
\sqrt{G_1^2+C_g^2+
4\left(1+\ln\frac{2}{\delta}\right)G_1^2},
\mu:=\frac{G_2^2}{c_\delta\sqrt T},
\eta:=\frac{\mu}{2G_2^2}$. Then Algorithm~\ref{algorithm: SGD with One Projection} outputs
$\widehat\theta=(\widehat\theta_0,\widehat\theta_1)\in K$ such that,
conditional on the training sample, with probability at least
$1-\delta$
\[
\widehat R_{\mathrm{amb},\phi}^{\mathrm{add}}
(f_{\widehat\theta_0},f_{\widehat\theta_1})
-
\widehat R_{\mathrm{amb},\phi}^{\mathrm{add}}
(f_{\widehat{\theta}_0^\star},f_{\widehat{\theta}_1^\star})
\lesssim
\frac{1}{\sqrt T}
\left[
\left(1+\frac{G_1^2}{\rho^2}\right)c_\delta
+\frac{G_1^2G_2^2}{\rho^2c_\delta}
+G_1\sqrt{\ln\frac{2}{\delta}}
\right].
\]
Consequently, under the assumptions of Theorem \ref{thm:empirical_surrogate_upper_bound}, taking
\[
T\gtrsim
\frac{1}{
\bigl(\epsilon_{\mathrm{amb},\phi}^{\mathrm{add}}(n,\delta)\bigr)^2}
\left[
\left(1+\frac{G_1^2}{\rho^2}\right)c_\delta
+\frac{G_1^2G_2^2}{\rho^2c_\delta}
+G_1\sqrt{\ln\frac{2}{\delta}}
\right]^2
\]
preserves its feasibility guarantees and yields its ambiguity
rate up to a constant factor, with an additional failure
probability of at most $\delta$.
\end{theorem}


\begin{algorithm2e}[t]
\small
\DontPrintSemicolon
\caption{\sffamily{SGD with One Projection}}
\label{algorithm: SGD with One Projection}
\textbf{Input:}
$\theta_0^0=\theta_1^0=0$, $\lambda_0=0$,
stepsize $\eta$, $\mu>0$,
and $\diamond\in\{\mathrm{add},\mathrm{prod}\}$.\\

\For{$t=0,\ldots,T-1$}{
    Sample $X_t$ uniformly from $\{X_i\}_{i=1}^n$ and Choose
$G_t^\diamond\in
\partial_{(\theta_0,\theta_1)}
\ell_{\mathrm{amb},\phi}^\diamond
(f_{\theta_0^t},f_{\theta_1^t};X_t)$
and
$u_t\in
\partial_{(\theta_0,\theta_1)}g(\theta_0^t,\theta_1^t)$
    \\
    $\theta'=(\theta_0^t,\theta_1^t)
    -\eta G_t^\diamond-\eta\lambda_tu_t$\\
    $(\theta_0^{t+1},\theta_1^{t+1})
    =\theta'/\max\{1,\|\theta'\|_2\}$\\
    $\lambda_{t+1}
    =[(1-\mu\eta)\lambda_t+\eta g(\theta_0^t,\theta_1^t)]_+$
}

$\widehat\theta=$ $\ell_2$ projection of
$\frac1T\sum_{t=0}^{T-1}(\theta_0^t,\theta_1^t)$
onto the constraint set
$\{(\theta_0,\theta_1):g(\theta_0,\theta_1)\leq0\}$.\\
\textbf{Output:} $\widehat\theta$
\end{algorithm2e}

\section{Experiments}
In this section, we evaluate the proposed approach to classification with abstention on several datasets. We consider both the additive and product surrogate formulations, using linear models and multilayer perceptrons (MLPs) to parameterize the class-specific score functions. We compare our approach with the method of \cite{lei2014classification}, which constructs classification regions by thresholding nonparametric estimates of conditional class probabilities. We report the ambiguity risk and class-conditional errors on held-out test data.

\subsection{Vertebral Column Dataset}

The Vertebral Column dataset \citep{barreto2005vertebral} contains 310 patients described by six biomechanical features, with 100 normal and 210 pathological cases. We use stratified splits of 217 training and 93 test observations, standardizing features using training-sample statistics. We compare a three-layer MLP with hinge loss under the additive and product formulations against Lei’s method \citep{lei2014classification}, with class-conditional error levels set to $\alpha_0=\alpha_1=0.05$. We report the mean and standard deviation of test ambiguity risk and class-conditional errors over ten runs.

As shown in Table~\ref{tab:vertebral-column}, both surrogate formulations
achieve lower mean ambiguity risk than Lei's method \citep{lei2014classification}, with class-conditional
errors close to the target level of $0.05$. The product formulation
achieves the lowest mean ambiguity risk ($0.2591$, compared with $0.3086$
for the additive formulation and $0.3505$ for Lei's method \citep{lei2014classification}), with slightly
higher mean class-conditional errors than the additive formulation.

\begin{table}[t]
\centering
\caption{Test performance on the Vertebral Column dataset \citep{barreto2005vertebral} with target
class-conditional error levels $\alpha_0=\alpha_1=0.05$.
We compare Lei's method \citep{lei2014classification} with three-layer MLPs trained using hinge loss
under the additive and product surrogate formulations.
Values are means $\pm$ standard deviations over ten runs.
$R_0$ and $R_1$ denote class-conditional errors, and
$R_{\mathrm{amb}}$ denotes ambiguity risk.}
\label{tab:vertebral-column}
\begin{tabular}{lccc}
\toprule
Method & $R_0$ & $R_1$ & $R_{\mathrm{amb}}$ \\
\midrule
Lei \citep{lei2014classification}
& $0.0333 \pm 0.0351$
& $0.0524 \pm 0.0290$
& $0.3505 \pm 0.1152$ \\
MLP, hinge, additive
& $0.0400 \pm 0.0344$
& $0.0524 \pm 0.0375$
& $0.3086 \pm 0.1153$ \\
MLP, hinge, product
& $0.0467 \pm 0.0358$
& $0.0540 \pm 0.0383$
& $0.2591 \pm 0.1029$ \\
\bottomrule
\end{tabular}
\end{table}

\subsection{MAGIC Gamma Telescope Dataset}

The MAGIC Gamma Telescope dataset \citep{bock2004magic} contains 19,020 simulated
atmospheric-shower events, each described by ten continuous features.
The task is to distinguish gamma showers (12,332 observations) from
hadronic showers (6,688 observations). We use a stratified split of
13,314 training and 5,706 test observations and compare three-layer
MLPs with hinge loss under the additive and product formulations against
Lei's method \citep{lei2014classification}, with $\alpha_0=\alpha_1=0.05$.

As shown in Table~\ref{tab:magic}, both formulations achieve lower
mean ambiguity risk than Lei's method \citep{lei2014classification}. The product formulation
performs better than the additive formulation in terms of mean ambiguity
risk ($0.4332$ versus $0.5107$), while keeping both mean
class-conditional errors below $0.05$.

\begin{table}[t]
\centering
\caption{Test performance on the MAGIC Gamma Telescope dataset \citep{bock2004magic}
with target class-conditional error levels $\alpha_0=\alpha_1=0.05$.
We compare Lei's method \citep{lei2014classification} with three-layer MLPs trained using hinge
loss under the additive and product surrogate formulations.
Values are means $\pm$ standard deviations over ten runs.
$R_0$ and $R_1$ denote class-conditional errors, and
$R_{\mathrm{amb}}$ denotes ambiguity risk.}
\label{tab:magic}
\begin{tabular}{lccc}
\toprule
Method & $R_0$ & $R_1$ & $R_{\mathrm{amb}}$ \\
\midrule
Lei \citep{lei2014classification}
& $0.0501 \pm 0.0046$
& $0.0485 \pm 0.0064$
& $0.5966 \pm 0.0197$ \\
MLP, hinge, additive
& $0.0316 \pm 0.0204$
& $0.0520 \pm 0.0232$
& $0.5107 \pm 0.0889$ \\
MLP, hinge, product
& $0.0418 \pm 0.0091$
& $0.0493 \pm 0.0092$
& $0.4332 \pm 0.0263$ \\
\bottomrule
\end{tabular}
\end{table}

\subsection{BCI Competition IV Dataset 2b}

The BCI Competition IV Dataset 2b \citep{leeb2008bci2b} contains EEG
recordings from nine subjects performing left- and right-hand
motor-imagery tasks. We report results for Subject 4, whose original
competition data comprise 400 training trials and 320 evaluation
trials. Our experiments use a stratified 70\%/30\% training--test
split, with preprocessing and feature parameters estimated from
the training data. We compare Lei's method with linear and MLP
models using hinge or squared hinge loss under the additive and
product formulations, with $\alpha_0=\alpha_1=0.05$.

As shown in Table~\ref{tab:bci-subject4}, the linear model with
hinge loss and the product formulation achieves a low mean ambiguity
risk of $0.0077$, with mean class-conditional errors of $0.0505$
and $0.0450$, approximately meeting the prescribed thresholds.
This compares favorably with the mean ambiguity risk of $0.2293$
achieved by the method of Lei \citep{lei2014classification}.

\begin{table}[t]
\centering
\caption{Test performance for Subject 4 of the BCI Competition IV
Dataset 2b \citep{leeb2008bci2b} with target class-conditional error levels
$\alpha_0=\alpha_1=0.05$. We compare Lei's method \citep{lei2014classification} with linear and
MLP models using hinge or squared hinge loss under the additive
and product formulations. Values are means $\pm$ standard deviations
over ten runs. $R_0$ and $R_1$ denote class-conditional errors,
and $R_{\mathrm{amb}}$ denotes ambiguity risk.}
\label{tab:bci-subject4}
\begin{tabular}{lccc}
\toprule
Method & $R_0$ & $R_1$ & $R_{\mathrm{amb}}$ \\
\midrule
Lei \citep{lei2014classification}
& $0.0351 \pm 0.0361$
& $0.0252 \pm 0.0163$
& $0.2293 \pm 0.4065$ \\
Linear, hinge, product
& $0.0505 \pm 0.0279$
& $0.0450 \pm 0.0147$
& $0.0077 \pm 0.0090$ \\
Linear, hinge, additive
& $0.0459 \pm 0.0242$
& $0.0450 \pm 0.0147$
& $0.0117 \pm 0.0098$ \\
MLP, hinge, product
& $0.0658 \pm 0.0321$
& $0.0396 \pm 0.0114$
& $0.0072 \pm 0.0085$ \\
MLP, hinge, additive
& $0.0523 \pm 0.0268$
& $0.0495 \pm 0.0196$
& $0.0050 \pm 0.0113$ \\
MLP, squared hinge, product
& $0.0459 \pm 0.0270$
& $0.0423 \pm 0.0113$
& $0.0158 \pm 0.0133$ \\
\bottomrule
\end{tabular}
\end{table}

\subsection{Phoneme Dataset}

Thew Phoneme dataset \citep{phoneme_openml} from OpenML contains 5,404 observations
described by five acoustic features, with the task of distinguishing
nasal from oral vowel sounds. We use a stratified 70\%/30\%
training--test split and compare MLP models with hinge loss under
the additive and product formulations against the method of
Lei \citep{lei2014classification}, with $\alpha_0=\alpha_1=0.05$.

As shown in Table~\ref{tab:phoneme}, the additive formulation
achieves mean class-conditional errors below both the prescribed
thresholds and those of Lei's method \citep{lei2014classification}, while also yielding
a lower mean ambiguity risk ($0.5292$ versus $0.5606$). The product formulation further reduces mean ambiguity risk to
$0.4482$, although its mean $R_0=0.0551$ slightly exceeds the
target level.

\begin{table}[t]
\centering
\caption{Test performance on the Phoneme dataset \citep{phoneme_openml} with target
class-conditional error levels $\alpha_0=\alpha_1=0.05$.
We compare the method of Lei \citep{lei2014classification} with MLP models trained
using hinge loss under the additive and product formulations.
Values are means $\pm$ standard deviations over ten runs.
$R_0$ and $R_1$ denote class-conditional errors, and
$R_{\mathrm{amb}}$ denotes ambiguity risk.}
\label{tab:phoneme}
\begin{tabular}{lccc}
\toprule
Method & $R_0$ & $R_1$ & $R_{\mathrm{amb}}$ \\
\midrule
Lei \citep{lei2014classification}
& $0.0531 \pm 0.0122$
& $0.0540 \pm 0.0069$
& $0.5606 \pm 0.0257$ \\
MLP, hinge, product
& $0.0551 \pm 0.0204$
& $0.0479 \pm 0.0151$
& $0.4482 \pm 0.0341$ \\
MLP, hinge, additive
& $0.0387 \pm 0.0308$
& $0.0462 \pm 0.0170$
& $0.5292 \pm 0.0607$ \\
\bottomrule
\end{tabular}
\end{table}

\subsection{Sensorless Drive Diagnosis Dataset}

The Sensorless Drive Diagnosis dataset \citep{bator2013sensorless} contains 48 features
extracted from motor-current signals under different operating
conditions. We consider classes 6 and 9, yielding 10,638 observations
equally divided between the two classes, with 7,446 training and
3,192 test observations. We compare MLP models with hinge loss
under the additive and product formulations against the method
of Lei \citep{lei2014classification}, with $\alpha_0=\alpha_1=0.05$.

As shown in Table~\ref{tab:sensorless}, the product formulation
achieves the lowest mean ambiguity risk ($0.0166$, compared with
$0.0644$ for the additive formulation and $0.2833$ for
Lei \citep{lei2014classification}). Both formulations yield mean $R_0$ below the
prescribed threshold, while their mean $R_1$ slightly exceeds it.

\begin{table}[t]
\centering
\caption{Test performance on the Sensorless Drive Diagnosis dataset \citep{bator2013sensorless}
(classes 6 versus 9) with target class-conditional error levels
$\alpha_0=\alpha_1=0.05$. We compare the method of Lei \citep{lei2014classification}
with MLP models trained using hinge loss under the additive and
product formulations. Values are means $\pm$ standard deviations
over ten runs. $R_0$ and $R_1$ denote class-conditional errors,
and $R_{\mathrm{amb}}$ denotes ambiguity risk.}
\label{tab:sensorless}
\begin{tabular}{lccc}
\toprule
Method & $R_0$ & $R_1$ & $R_{\mathrm{amb}}$ \\
\midrule
Lei \citep{lei2014classification}
& $0.0538 \pm 0.0093$
& $0.0505 \pm 0.0059$
& $0.2833 \pm 0.1541$ \\
MLP, hinge, product
& $0.0437 \pm 0.0154$
& $0.0543 \pm 0.0116$
& $0.0166 \pm 0.0123$ \\
MLP, hinge, additive
& $0.0441 \pm 0.0203$
& $0.0527 \pm 0.0127$
& $0.0644 \pm 0.0527$ \\
\bottomrule
\end{tabular}
\end{table}

\subsection{Synthetic Gaussian Data}

We consider a binary classification problem in dimension $d=10$,
with class-conditional distributions
\[
X\mid Y=0 \sim \mathcal{N}(\mathbf{0}_{10},I_{10}),
\qquad
X\mid Y=1 \sim \mathcal{N}(0.5\mathbf{1}_{10},2I_{10}).
\]
We generate 10,000 observations, equally divided between the two
classes, and use a stratified split of 7,000 training and 3,000
test observations. We compare MLP models using hinge and squared
hinge losses under the additive and product formulations against
the method of Lei \citep{lei2014classification}, with $\alpha_0=\alpha_1=0.05$.

As shown in Table~\ref{tab:gaussian}, the additive formulation with
hinge loss achieves mean class-conditional errors below both
prescribed thresholds and those of Lei's method \citep{lei2014classification}, while
also yielding a lower mean ambiguity risk ($0.5056$ versus
$0.6038$). The product formulation with squared hinge loss
reduces mean ambiguity risk to $0.4415$, with mean $R_0=0.0512$
slightly above the target and mean $R_1=0.0479$ below it.

\begin{table}[t]
\centering
\caption{Test performance on the 10-dimensional Gaussian dataset
with $X\mid Y=0\sim\mathcal{N}(\mathbf{0}_{10},I_{10})$ and
$X\mid Y=1\sim\mathcal{N}(0.5\mathbf{1}_{10},2I_{10})$,
using 7,000 training observations and target error levels
$\alpha_0=\alpha_1=0.05$. Values are means $\pm$ standard
deviations over ten runs. $R_0$ and $R_1$ denote class-conditional
errors, and $R_{\mathrm{amb}}$ denotes ambiguity risk.}
\label{tab:gaussian}
\begin{tabular}{lccc}
\toprule
Method & $R_0$ & $R_1$ & $R_{\mathrm{amb}}$ \\
\midrule
Lei \citep{lei2014classification}
& $0.0485 \pm 0.0051$
& $0.0510 \pm 0.0104$
& $0.6038 \pm 0.0207$ \\
MLP, hinge, product
& $0.0611 \pm 0.0092$
& $0.0519 \pm 0.0103$
& $0.4089 \pm 0.0392$ \\
MLP, hinge, additive
& $0.0473 \pm 0.0173$
& $0.0490 \pm 0.0195$
& $0.5056 \pm 0.0921$ \\
MLP, squared hinge, product
& $0.0512 \pm 0.0109$
& $0.0479 \pm 0.0098$
& $0.4415 \pm 0.0441$ \\
MLP, squared hinge, additive
& $0.0655 \pm 0.0193$
& $0.0496 \pm 0.0174$
& $0.4565 \pm 0.0754$ \\
\bottomrule
\end{tabular}
\end{table}

\section{Conclusion}

We studied classification with abstention under separate class-conditional
error constraints, treating abstention as the price of controlling both
types of error. We characterized the distribution-free minimax rate of
excess ambiguity risk up to logarithmic factors. To accommodate
models such as neural networks, we introduced additive and product
surrogate formulations and derived statistical guarantees for constraint
violations and excess surrogate ambiguity risk. We further characterized
the computational complexity of a stochastic primal-dual algorithm in
the convex setting. Experiments on synthetic and real-world datasets showed that the proposed
approach can substantially reduce ambiguity compared with the plug-in
method of Lei~\cite{lei2014classification}, while achieving
class-conditional errors close to the prescribed levels. Extending the computational guarantees to
nonconvex models and relating excess surrogate ambiguity to excess
0--1 ambiguity remain directions for future work.

\newpage
\bibliography{references}
\bibliographystyle{ICLR}

\newpage
\appendix

\section*{Appendix}

\section{Proof of Theorem~\ref{thm:upper_bound} [Upper bound]}
\label{app:proof-main}

\begin{proof}
Let $\mathcal E_n:=\{n_0>0,n_1>0\}$. For $j\in\{0,1\}$, define
\[
Z_j
=
\sup_{h_j\in\mathcal H}
\left|
\widehat R_j(h_j)-R_j(h_j)
\right|.
\]

For every integer $m\ge 1$, conditional on the event $\{n_j=m\}$, the
observations $\{X_i:Y_i=j\}$ are i.i.d. from $P_j$. Hence, there exists
a universal constant $C>0$ such that
\[
\mathbb P\left(
Z_j
\le
C
\sqrt{
\frac{
d_{\mathcal{H}}\log(em)+\log(1/\delta)
}{
m
}
}
\;\middle|\;
n_j=m
\right)
\ge 1-\delta .
\]
Increasing $C$ if necessary, define the random radius
\[
\epsilon_j(n_j,\delta)
=
C
\sqrt{
\frac{
d_{\mathcal{H}}\log(en_j)+\log(1/\delta)
}{
n_j
}
},
\qquad n_j\ge 1.
\]
Then, for each $j\in\{0,1\}$,
\[
\begin{aligned}
\mathbb P\left(
Z_j>\frac{\epsilon_j(n_j,\delta)}{2}
\;\middle|\;
\mathcal E_n
\right)
&=
\sum_{m=1}^{n-1}
\mathbb P\left(
Z_j>\frac{\epsilon_j(m,\delta)}{2}
\;\middle|\;
n_j=m
\right)
\mathbb P\left(
n_j=m
\;\middle|\;
\mathcal E_n
\right)\le
\sum_{m=1}^{n-1}
\delta\,
\mathbb P\left(
n_j=m
\;\middle|\;
\mathcal E_n
\right)
=
\delta .
\end{aligned}
\]
Define the event
\[
\mathcal A
:=
\mathcal E_n
\cap
\left\{
Z_j
\le
\frac{\epsilon_j(n_j,\delta)}{2}
\text{ for } j=0,1
\right\}.
\]
Therefore, by a union bound, we have $
    \mathbb P(\mathcal A\mid \mathcal E_n)
    \ge
    1-2\delta$. On the event $\mathcal A$, since
$(\widehat h_0,\widehat h_1)$ is feasible for the empirical problem,
i.e.,
\[
     \widehat R_j(\widehat h_j)
    \leq
    \alpha_j+\frac{\epsilon_j(n_j,\delta)}{2},
    \qquad j\in\{0,1\},
\]
we have
\begin{align}\label{constraints_bounds}
    R_j(\widehat h_j)
    \leq
    \widehat R_j(\widehat h_j)+Z_j
    \leq
    \alpha_j+\epsilon_j(n_j,\delta),
    \qquad j\in\{0,1\}.
\end{align}

It remains to control the ambiguity term. Let
\[
\mathcal{G}
=
\left\{
g_{h_0,h_1}:x\mapsto h_0(x)h_1(x)
:
h_0,h_1\in\mathcal H
\right\}.
\]
For a function $g\in \mathcal{G}$, we write
$P_X g:=\mathbb E_{X\sim P_X}[g(X)]$ and
$P_n g:=n^{-1}\sum_{i=1}^n g(X_i)$. Now observe that
\[
R_{\mathrm{amb}}(h_0,h_1)
=
P_X\bigl(h_0(X)=1,\ h_1(X)=1\bigr)
=
P_X g_{h_0,h_1}.
\]
Similarly, $
\widehat R_{\mathrm{amb}}(h_0,h_1)
=
P_n g_{h_0,h_1}$. Therefore,
\begin{align}\label{equivalence_ambigiuty}
\sup_{h_0,h_1\in\mathcal H}
\left|
R_{\mathrm{amb}}(h_0,h_1)
-
\widehat R_{\mathrm{amb}}(h_0,h_1)
\right|
=
\sup_{g\in\mathcal G}
\left|
P_X g - P_n g
\right|.
\end{align}

By Lemma~\ref{lemma:VC_bound}, the class $\mathcal G$ satisfies
$\operatorname{VC}(\mathcal G)\leq C d_{\mathcal{H}}$, for a universal
constant $C>0$. Let
\[
\Delta_{\mathrm{amb}}(n,\delta)
:=
C
\sqrt{
\frac{
d_{\mathcal{H}}\log(en)+\log(1/\delta)
}{
n
}
},
\]
and define the event
\[
\mathcal B
:=
\left\{
\sup_{h_0,h_1\in\mathcal H}
\left|
R_{\mathrm{amb}}(h_0,h_1)
-
\widehat R_{\mathrm{amb}}(h_0,h_1)
\right|
\leq
\Delta_{\mathrm{amb}}(n,\delta)
\right\}.
\]
Then, by the uniform convergence bound over the class $\mathcal G$ and
\eqref{equivalence_ambigiuty}, we obtain $
    \mathbb P(\mathcal B^c)\leq \delta $. Hence, conditional on $\mathcal E_n$,
\[
    \mathbb P(\mathcal B^c\mid \mathcal E_n)
    =
    \frac{\mathbb P(\mathcal B^c\cap \mathcal E_n)}
    {\mathbb P(\mathcal E_n)}
    \leq
    \frac{\delta}{\mathbb P(\mathcal E_n)}.
\]

Now we consider the event $\mathcal A\cap\mathcal B$. Since
$(h_0^*,h_1^*)$ is feasible for the population problem,
\[
    R_j(h_j^*)\leq \alpha_j,
    \qquad j\in\{0,1\}.
\]
Using $\mathcal A$, we obtain
\[
    \widehat R_j(h_j^*)
    \leq
    R_j(h_j^*)+Z_j
    \leq
    \alpha_j+\frac{\epsilon_j(n_j,\delta)}{2},
    \qquad j\in\{0,1\}.
\]
Thus $(h_0^*,h_1^*)$ is feasible for the empirical problem. Therefore,
by the optimality of $(\widehat h_0,\widehat h_1)$ in the empirical
problem,
\[
    \widehat R_{\mathrm{amb}}(\widehat h_0,\widehat h_1)
    \leq
    \widehat R_{\mathrm{amb}}(h_0^*,h_1^*).
\]
Using $\mathcal B$, we obtain
\[
\begin{aligned}
R_{\mathrm{amb}}(\widehat h_0,\widehat h_1)
-
R_{\mathrm{amb}}(h_0^*,h_1^*)
&\leq
\left[
R_{\mathrm{amb}}(\widehat h_0,\widehat h_1)
-
\widehat R_{\mathrm{amb}}(\widehat h_0,\widehat h_1)
\right] +
\left[
\widehat R_{\mathrm{amb}}(h_0^*,h_1^*)
-
R_{\mathrm{amb}}(h_0^*,h_1^*)
\right]\\
&\leq
2\Delta_{\mathrm{amb}}(n,\delta).
\end{aligned}
\]
Since $R_{\mathrm{amb}}(h_0^*,h_1^*)=R_{\mathrm{amb}}^*$, this gives
\begin{align}\label{excess_ambiguity_bound}
    R_{\mathrm{amb}}(\widehat h_0,\widehat h_1)
    -
    R_{\mathrm{amb}}^*
    \leq
    2\Delta_{\mathrm{amb}}(n,\delta).
\end{align}

Finally,
\[
\begin{aligned}
\mathbb P(\mathcal A\cap\mathcal B\mid \mathcal E_n)
&\geq
1-\mathbb P(\mathcal A^c\mid \mathcal E_n)
  -\mathbb P(\mathcal B^c\mid \mathcal E_n)\geq
1-2\delta-\frac{\delta}{\mathbb P(\mathcal E_n)}.
\end{aligned}
\]
Therefore, conditional on $\mathcal E_n$, with probability at least $
    1-2\delta-\frac{\delta}{\mathbb P(\mathcal E_n)}$, the two class-conditional bounds in~\eqref{constraints_bounds} and the
excess ambiguity bound in~\eqref{excess_ambiguity_bound} hold
simultaneously.
\end{proof}

\section{Proof of Proposition \ref{prop:strict_feasibility_impossibility}}
\begin{proof}
Let $\mathcal X=\{a,b\}$, and define $g_0(a)=0$, $g_0(b)=1$, $g_1(a)=1$, and $g_1(b)=0$. Let $\mathcal H=\{\mathbf 1,\mathbf 0,g_0,g_1\}$. VC dimension of this class is $2$. Consider the distribution with equal class priors $\pi_0=\pi_1=1/2$ and class-conditional distributions $P_0(X=a)=\alpha_0$, $P_0(X=b)=1-\alpha_0$, $P_1(X=a)=1-\alpha_1$, and $P_1(X=b)=\alpha_1$. We have $
    R_0(g_0)=\alpha_0$ and
    $R_1(g_1)=\alpha_1$. Therefore, the pair $(g_0,g_1)$ is feasible.
Moreover,
$R_{\mathrm{amb}}(g_0,g_1)=0$. It follows that $
    R_{\mathrm{amb}}^*=0$.

The class-$0$ risks of the classifiers in $\mathcal{H}$ are
\[
    R_0(\mathbf 1)=0,
    \qquad
    R_0(g_0)=\alpha_0,
    \qquad
    R_0(g_1)=1-\alpha_0,
    \qquad
    R_0(\mathbf 0)=1.
\]
Hence, the only classifier satisfying $R_0(h_0)<\alpha_0$ is $h_0=\mathbf 1$. Similarly, the only classifier satisfying $R_1(h_1)<\alpha_1$ is $h_1=\mathbf 1$.
Thus, the only pair satisfying both strict constraints is $
    (h_0,h_1)=(\mathbf 1,\mathbf 1).$ Its ambiguity risk is $
    R_{\mathrm{amb}}(\mathbf 1,\mathbf 1)=1.$ Therefore,
$
    \mathcal E_{\mathrm{amb}}(\mathbf 1,\mathbf 1)
    =
    1-R_{\mathrm{amb}}^*
    =1,$ which proves the result.
\end{proof}

\section{Proof of Theorem~\ref{thm:randomized_upper_bound} [Upper bound for the randomized learner]}\label{app:proof-upper-bound-randomized}

\begin{proof}

Conditioning on $\mathcal E_n$, let $\mathcal B$ denote the event on which the conclusions of
Theorem~\ref{thm:upper_bound} hold for the empirical solution
$(\widehat h_0,\widehat h_1)$. Thus, we have $
\mathbb P(\mathcal B\mid \mathcal E_n)
\geq
1-2\delta-\frac{\delta}{\mathbb P(\mathcal E_n)}$. On $\mathcal B$, we have
\[
R_j(\widehat h_j)\leq \alpha_j+\epsilon_j(n_j,\delta),
\qquad j\in\{0,1\},
\]
and
\begin{align}\label{amb_bound_proof_thm2}
\mathcal E_{\mathrm{amb}}(\widehat h_0,\widehat h_1)
\leq
C
\sqrt{
\frac{
d_{\mathcal H}\log(e n)+\log(1/\delta)
}{
n
}} .
\end{align}

$\mathbb E_\zeta$ denotes expectation only with respect to the
internal randomness, with the data $S^n$ fixed. Since
the constant-one classifier has class-conditional error equal to zero,
we have
\[
\overline R_j(\widetilde h_j)
=
\mathbb E_{\zeta}
\left[
R_j\bigl(B_j\widehat h_j+(1-B_j)\mathbf 1\bigr)
\right]
=
q_jR_j(\widehat h_j).
\]
Therefore,
\[
\overline R_j(\widetilde h_j)
\leq
\frac{\alpha_j}{\alpha_j+\epsilon_j(n_j,\delta)}
\bigl(\alpha_j+\epsilon_j(n_j,\delta)\bigr)
=
\alpha_j,
\qquad j\in\{0,1\}.
\]

We next bound the ambiguity risk. For every realization of
$\zeta=(B_0,B_1)$,
\[
R_{\mathrm{amb}}(\widetilde h_0,\widetilde h_1)
\leq
R_{\mathrm{amb}}(\widehat h_0,\widehat h_1)
+
\mathbf 1\{B_0=0\}
+
\mathbf 1\{B_1=0\}.
\]
Taking expectation with respect to $\zeta$ gives
\[
\overline R_{\mathrm{amb}}(\widetilde h_0,\widetilde h_1)
\leq
R_{\mathrm{amb}}(\widehat h_0,\widehat h_1)
+
(1-q_0)+(1-q_1).
\]
Hence, 
\begin{align}\label{amb_bound_epsilon}
\overline{\mathcal E}_{\mathrm{amb}}
(\widetilde h_0,\widetilde h_1)
\leq
\mathcal E_{\mathrm{amb}}(\widehat h_0,\widehat h_1)
+
\sum_{j=0}^1(1-q_j)\leq \mathcal E_{\mathrm{amb}}(\widehat h_0,\widehat h_1)
+ \sum_{j=0}^1 \frac{\epsilon_j(n_j,\delta)}{\alpha_j}.
\end{align}

Next, define the event $
\mathcal N_n
:=
\left\{
n_j\geq \frac{\pi_j n}{2},\ j\in\{0,1\}
\right\}$. Since $n_j\sim \operatorname{Binomial}(n,\pi_j)$ and $\pi_j>0$, a
standard binomial concentration bound gives
\[
\mathbb P(\mathcal N_n^c)
\leq
2\exp\left(-c n\min\{\pi_0,\pi_1\}\right),
\]
for a universal constant $c>0$. Therefore,
\[
\mathbb P(\mathcal N_n^c\mid \mathcal E_n)
\leq
\frac{
2\exp\left(-c n\min\{\pi_0,\pi_1\}\right)
}{
\mathbb P(\mathcal E_n)
}.
\]
On $\mathcal N_n$, for each $j\in\{0,1\}$,
\[
\epsilon_j(n_j,\delta)
=
C
\sqrt{
\frac{
d_{\mathcal H}\log(e n_j)+\log(1/\delta)
}{
n_j
}}
\leq
C_{\pi}
\sqrt{
\frac{
d_{\mathcal H}\log(e n)+\log(1/\delta)
}{
n
}} .
\]
Combining the bounds \eqref{amb_bound_proof_thm2} and \eqref{amb_bound_epsilon}, on $\mathcal B\cap\mathcal N_n$,
\[
\overline{\mathcal E}_{\mathrm{amb}}
(\widetilde h_0,\widetilde h_1)
\leq
C_{\alpha,\pi}
\sqrt{
\frac{
d_{\mathcal H}\log(e n)+\log(1/\delta)
}{
n
}} .
\]
Finally, conditional on $\mathcal E_n$,
\[
\mathbb P(\mathcal B\cap\mathcal N_n\mid \mathcal E_n)
\geq
1
-
2\delta
-
\frac{\delta}{\mathbb P(\mathcal E_n)}
-
\frac{
2\exp\left(-c n\min\{\pi_0,\pi_1\}\right)
}{
\mathbb P(\mathcal E_n)
}.
\]
This concludes the proof.
\end{proof}

\section{Proof of Theorem~\ref{thm:minimax_lower_bound_randomized} [Minimax Lower bound]}
\label{app:proof-lower-bound}
We will use the following proposition to construct a well-separated packing of the parameter space.
\begin{proposition}[Gilbert--Varshamov bound]\label{Varshamov}
Let \(m\ge 8\). Then there exists a subset \(\{\sigma_0,\ldots,\sigma_M\}\) of \(\{-1,+1\}^m\) such that \(\sigma_0=(1,1,\ldots,1)\),
\[
d_H(\sigma_j,\sigma_k)\ge \frac{m}{8},
\qquad \forall\,0\le j<k\le M,
\qquad\text{and}\qquad
M\ge 2^{m/8},
\]
where \(d_H(\sigma,\sigma')=\#\{i\in[m]:\sigma(i)\neq\sigma'(i)\}\) is the Hamming distance.
\end{proposition}
We will also use the following proposition to convert a packing into a lower bound on the probability of estimation error. This proposition is a direct consequence of the Fano/Tsybakov minimax lower bound \cite[Theorem~2.5]{Tsybakov:1315296}, proved in Section \ref{proof_Auxiliary_Results}.

\begin{proposition}\label{minimax_tsybakov}
Let $\{P_\theta:\theta\in\Theta\}$ be a statistical model, and let
$S^n$ denote a sample drawn from $P_\theta^{\otimes n}$. Assume that $M \ge 2$ and that the function $\operatorname{dist}(\cdot,\cdot)$ is a semi-metric.
Furthermore, suppose that $\{\Pi_{\theta_j}\}_{\theta_j \in \Theta}$ is a family of distributions indexed over the parameter space $\Theta$, and that $\Theta$ contains elements
$\theta_0,\theta_1,\ldots,\theta_M$ such that:
\begin{enumerate}
    \item[(i)] 
   
    $\operatorname{dist}(\theta_i,\theta_j) \ge 2s > 0, \forall\, 0 \le i < j \le M$.

    \item[(ii)]
    $
    \Pi_{\theta_j} \ll \Pi_{\theta_0}, \forall\, j=1,\ldots,M,
    $ and
    $
    \frac{1}{M}\sum_{j=1}^M
D_{\mathrm{KL}}\!\left(\Pi_{\theta_j} \,\middle\|\, \Pi_{\theta_0}\right)
    \le \eta \log M$ with  $0<\eta<\frac18,
    $ where 
    $D_{\mathrm{KL}}$ denotes the KL divergence.
\end{enumerate}
Define$
L_{M,\eta}
:=
\frac{\sqrt M}{1+\sqrt M}
\left(
1-2\eta-\sqrt{\frac{2\eta}{\log M}}
\right).$ Then
\[
\inf_{\widetilde\theta}
\sup_{\theta\in\Theta}
P_\theta^{\otimes n}
\left(
\mathbb E_\zeta
\left[
d\bigl(\widetilde\theta(S^n,\zeta),\theta\bigr)
\right]
\ge
\frac{sL_{M,\eta}}{2}
\right)
\ge
\frac{L_{M,\eta}}{2-L_{M,\eta}},
\]

where the infimum is over all randomized estimators of the form $
\widetilde\theta=\widetilde\theta(S^n,\zeta)$, with $\zeta$ denoting auxiliary randomness of the estimator, independent
of the sample $S^n$.
\end{proposition}
We divide the proof into two cases: $
\text{(I) } d_{\mathcal H}\geq 18$ and
$\text{(II) } 4\le d_{\mathcal H}\le 17.$

\noindent\textbf{Case (I): \boldmath{$d_{\mathcal H}\geq 18$.}} Let $m=\lfloor \frac{d_{\mathcal{H}}-2}{2} \rfloor$. We choose $2m+2$ points $x_0,x_1,x_{1,1},\ldots,x_{1,m},x_{2,1},\ldots,x_{2,m}\in\mathcal X$ that are shattered by $\mathcal H$. Let $\sigma=(\sigma_1,...,\sigma_m)\in \{-1,+1\}^m$. For each $i\in[m]$, define
\[
(a_i^{\sigma},b_i^{\sigma})
:=
\begin{cases}
(x_{1,i},x_{2,i}), & \text{if } \sigma_i=1,\\
(x_{2,i},x_{1,i}), & \text{if } \sigma_i=-1.
\end{cases}
\]

\noindent \textbf{Construction of the distributions:} Fix a small number $\gamma\in(0,\frac{4}{100})$, to be chosen later as $\gamma=c\sqrt{\frac{m}{n}}$ for a sufficiently small universal constant $c>0$. For every $\sigma \in \{-1,+1\}^m$, define $P^{\sigma}$ as follows. The label prior is \[P^{\sigma}(Y=0)=P^{\sigma}(Y=1)=\frac{1}{2}.\]
The conditional distribution $P_0^\sigma := P^\sigma_{X\mid Y=0}$ is defined by $P_0^\sigma(x_0)=1-\alpha_0(2+\gamma)$, $P_0^\sigma(x_1)=0$, and, for each $i\in[m]$, 
\[
P_0^\sigma(a_i^\sigma)=\frac{\alpha_0(1+\gamma)}{m},
\qquad
P_0^\sigma(b_i^\sigma)=\frac{\alpha_0}{m}.
\]
Similarly, define $P_1^\sigma= X\mid Y=1$ by $P_1^\sigma(x_1)=1-\alpha_1(2+\gamma)$, $P_1^{\sigma}(x_0)=0$ , and, for each $i\in [m]$, 
\[
P_1^\sigma(b_i^\sigma)=\frac{\alpha_1(1+\gamma)}{m},
\qquad
P_1^\sigma(a_i^\sigma)=\frac{\alpha_1}{m}.
\]

Because $\alpha_0,\alpha_1<\frac{49}{100}$ and $\gamma<\frac{4}{100}$, all masses are nonnegative.

\noindent \textbf{Oracle pair has zero ambiguity:} Because the selected points are shattered by $\mathcal{H}$, there exist classifiers $h^*_{0,\sigma},h^*_{1,\sigma}\in \mathcal{H}$ such that
\[
h_{0,\sigma}^*(x_0)=1,\quad
h_{0,\sigma}^*(x_1)=0,\quad
h_{0,\sigma}^*(a_i^\sigma)=1,\quad
h_{0,\sigma}^*(b_i^\sigma)=0,\quad \forall i\in[m],
\]

and 

\[
h_{1,\sigma}^*(x_0)=0,\quad
h_{1,\sigma}^*(x_1)=1,\quad
h_{1,\sigma}^*(a_i^\sigma)=0,\quad
h_{1,\sigma}^*(b_i^\sigma)=1,\quad \forall i\in[m].
\]

For $h^*_{0,\sigma}$,
\[
R_0^\sigma(h_{0,\sigma}^*)=P_0^\sigma\bigl(h_{0,\sigma}^*(X)=0\bigr)
=\sum_{i=1}^{m} P_0^\sigma(b_i^\sigma)=\alpha_0.
\]
Similarly, for $h^*_{1,\sigma}$,
\[
R_1^\sigma(h_{1,\sigma}^*)=\sum_{i=1}^{m}P_1^\sigma(a_i^\sigma)=\alpha_1.
\]
Thus the pair is feasible. Moreover, we have $h_{0,\sigma}^*(x)h_{1,\sigma}^*(x)=0$, which implies that $R_{\mathrm{amb}}^{*,\sigma}=R_{\mathrm{amb}}^\sigma(h_{0,\sigma}^*,h_{1,\sigma}^*)=0.$

\noindent \textbf{Ambiguity--Hamming reduction:}
Let $
    (S^n,\zeta)\mapsto (\widetilde h_0,\widetilde h_1)
    \in \mathcal H^+\times \mathcal H^+$ be a randomized learning rule, where \(\zeta\) is independent of \(S^n\). We next define a decoding map that maps each pair \((\widetilde h_0,\widetilde h_1)\in\mathcal H^+\times\mathcal H^+\) to a sign vector \(\widehat{\sigma}(\widetilde h_0,\widetilde h_1)\in\{-1,+1\}^m\). This allows us to relate the ambiguity of \((\widetilde h_0,\widetilde h_1)\) to the Hamming distance between \(\widehat{\sigma}(\widetilde h_0,\widetilde h_1)\) and the underlying parameter \(\sigma\).

For a fixed sample \(S^n\), a fixed auxiliary random variable \(\zeta\), and a
coordinate \(i\in[m]\), define
\[
A_i^\zeta:=\widetilde h_0(a_i^\sigma),
\qquad
B_i^\zeta:=\widetilde h_0(b_i^\sigma),
\qquad
C_i^\zeta:=\widetilde h_1(a_i^\sigma),
\qquad
D_i^\zeta:=\widetilde h_1(b_i^\sigma).
\]
The oracle pattern at coordinate $i$ is $(A_i,B_i,C_i,D_i)=(1,0,0,1)$. Define the overlap count $O_i^\zeta=A_i^\zeta C_i^\zeta+B_i^\zeta D_i^\zeta.$ Let
\[
O^\zeta=\sum_{i=1}^{m}O_i^\zeta,\qquad M_0^\zeta=\sum_{i=1}^{m}(1-A_i^\zeta),\qquad M_1^\zeta=\sum_{i=1}^{m}(1-D_i^\zeta),\qquad E_0^\zeta=\sum_{i=1}^{m}B_i^\zeta,\qquad E_1^\zeta=\sum_{i=1}^{m}C_i^\zeta.
\]

We write
\[
\overline O:=\mathbb E_\zeta[O^\zeta],
\qquad
\overline M_0:=\mathbb E_\zeta[M_0^\zeta],
\qquad
\overline M_1:=\mathbb E_\zeta[M_1^\zeta],
\qquad
\overline E_0:=\mathbb E_\zeta[E_0^\zeta],
\qquad
\overline E_1:=\mathbb E_\zeta[E_1^\zeta].
\]
Consider the feasibility event
\[
\overline{\mathcal E}_\sigma
:=
\left\{
    \overline R_0^\sigma(\widetilde h_0)\le \alpha_0+r_0,
    \quad
    \overline R_1^\sigma(\widetilde h_1)\le \alpha_1+r_1
\right\}.
\]
Let $\Pi^\sigma := \left(P^\sigma\right)^{\otimes n}$ denote the distribution of the sample $S^{n}$. By the definition of an \((r_0,r_1,\delta)\)-approximate randomized
\((\alpha_0,\alpha_1)\)-learner, we have $
    \Pi^\sigma(\overline{\mathcal E}_\sigma)\ge 1-\delta$. On \(\overline{\mathcal E}_\sigma\), the constraint on $\widetilde h_0$ gives $\mathbb E_\zeta P_0^{\sigma}(h_0(X)=1)\geq 1-\alpha_0-r_0$. On the other hand, the oracle satisfies $P_0^{\sigma}(h^*_{0,\sigma}(X)=1)=1-\alpha_0$. Therefore,
\begin{align}\label{difference_mass}
    P_0^\sigma(h^*_{0,\sigma}(X)=1)
    -
    \mathbb E_\zeta P_0^\sigma(\widetilde h_0(X)=1)
    \le r_0.
\end{align}
Now, by the definitions of $M_0^\zeta$ and $E_0^\zeta$, we have $M_0^\zeta=\#\{i:\widetilde h_0(a_i^{\sigma})=0\}$ and $E_0^\zeta=\#\{i:\widetilde h_0(b_i^{\sigma})=1\}$. Compared with the oracle $h_{0,\sigma}^*$, every missed point $a_i^{\sigma}$ loses mass $\frac{\alpha_0(1+\gamma)}{m}$ , and every extra accepted point $b_i^{\sigma}$ gains mass $\frac{\alpha_0}{m}$. Hence, \eqref{difference_mass} implies that
\[
    \frac{\alpha_0}{m}
    \left(
        (1+\gamma)\overline M_0-\overline E_0
    \right)
    \le r_0,
\]
or equivalently
\begin{align}\label{eq:rand-eq1}
    (1+\gamma)\overline M_0-\overline E_0
    \le
    \frac{m r_0}{\alpha_0}.
\end{align}
Similarly, the class-\(1\) constraint gives
\begin{align}
    (1+\gamma)\overline M_1-\overline E_1
    \le
    \frac{m r_1}{\alpha_1}.
    \label{eq:rand-eq2}
\end{align}
For every \(\zeta\), we have the pointwise inequalities $
    B_i^\zeta
    \le
    B_i^\zeta D_i^\zeta+(1-D_i^\zeta)$, and $
    C_i^\zeta
    \le
    A_i^\zeta C_i^\zeta+(1-A_i^\zeta)$. Summing over \(i\) and taking expectation with respect to \(\zeta\), we obtain
\begin{align}
    \overline E_0
    \le
    \overline O+\overline M_1,
    \label{eq:rand-eq3}
\end{align}
and
\begin{align}
    \overline E_1
    \le
    \overline O+\overline M_0.
    \label{eq:rand-eq4}
\end{align}
Combining \eqref{eq:rand-eq1} with \eqref{eq:rand-eq3}, we get
\[
    (1+\gamma)\overline M_0
    \le
    \overline O+\overline M_1+\frac{m r_0}{\alpha_0}.
\]
Combining \eqref{eq:rand-eq2} with \eqref{eq:rand-eq4}, we get
\[
    (1+\gamma)\overline M_1
    \le
    \overline O+\overline M_0+\frac{m r_1}{\alpha_1}.
\]
Adding these two inequalities yields
\begin{align}
    \overline M_0+\overline M_1
    \le
    \frac{2\overline O}{\gamma}
    +
    \frac{m}{\gamma}
    \left(
        \frac{r_0}{\alpha_0}
        +
        \frac{r_1}{\alpha_1}
    \right).
    \label{eq:rand-M-bound}
\end{align}
For each realization of \(\zeta\), define the decoded sign vector $
    \widetilde\sigma
    =
    \widehat\sigma(\widetilde h_0,\widetilde h_1)
    \in\{-1,+1\}^m$ as follows: at coordinate \(i\), choose the sign
whose oracle pattern is closer to the induced pattern of
\((\widetilde h_0,\widetilde h_1)\), breaking ties arbitrarily. If \(\widetilde\sigma_i\neq\sigma_i\), then the induced pattern cannot equal
the true oracle pattern. Hence, at
least one of the following must occur: 
\[
A_i^\zeta=0,\qquad D_i^\zeta=0,\qquad A_i^\zeta C_i^\zeta+B_i^\zeta D_i^\zeta\ge 1.
\]
Therefore, for every \(\zeta\),
\[
    \mathbf 1\{\widetilde\sigma_i\neq\sigma_i\}
    \le
    (1-A_i^\zeta)+(1-D_i^\zeta)
    +
    A_i^\zeta C_i^\zeta+B_i^\zeta D_i^\zeta .
\]
Summing over \(i\) and taking expectation over \(\zeta\), we obtain
\[
    \mathbb E_\zeta
    \left[
        d_H(\widetilde\sigma,\sigma)
    \right]
    \le
    \overline M_0+\overline M_1+\overline O.
\]
Using \eqref{eq:rand-M-bound}, this gives
\[
    \mathbb E_\zeta
    \left[
        d_H(\widetilde\sigma,\sigma)
    \right]
    \le
    \frac{3\overline O}{\gamma}
    +
    \frac{m}{\gamma}
    \left(
        \frac{r_0}{\alpha_0}
        +
        \frac{r_1}{\alpha_1}
    \right).
\]
Thus,
\begin{align}\label{eq:rand-lower-bound-O}
\boxed{
    \overline O
    \ge
    \frac{\gamma}{3}
    \mathbb E_\zeta
    \left[
        d_H(\widetilde\sigma,\sigma)
    \right]
    -
    \frac{m}{3}
    \left(
        \frac{r_0}{\alpha_0}
        +
        \frac{r_1}{\alpha_1}
    \right).}
\end{align}
This is the key reduction: a large Hamming distance implies a large overlap, unless the learning rule
constraints are violated.

Now we translate the averaged overlap count into averaged ambiguity risk. Each
point among the paired points has marginal mass at least $
    \frac{c_\alpha}{m}$, where
 $c_\alpha:=\frac12\min\{\alpha_0,\alpha_1\}$.
Therefore,
\[
    \overline R_{\mathrm{amb}}^\sigma(\widetilde h_0,\widetilde h_1)
    =
    \mathbb E_\zeta
    \left[
        R_{\mathrm{amb}}^\sigma(\widetilde h_0,\widetilde h_1)
    \right]
    \ge
    \frac{c_\alpha}{m}\overline O.
\]
Combining this with \eqref{eq:rand-lower-bound-O}, we get
\begin{align}
\boxed{
    \overline R_{\mathrm{amb}}^\sigma(\widetilde h_0,\widetilde h_1)
    \ge
    c_\alpha
    \left[
        \frac{\gamma}{3m}
        \mathbb E_\zeta
        \left[
            d_H(\widetilde\sigma,\sigma)
        \right]
        -
        \frac13
        \left(
            \frac{r_0}{\alpha_0}
            +
            \frac{r_1}{\alpha_1}
        \right)
    \right].
}
    \label{eq:rand-key-reduction-amb}
\end{align}
This is the ambiguity-to-Hamming reduction.

\noindent \textbf{Packing by Gilbert–Varshamov:} By Proposition \ref{Varshamov}, there exists $\Sigma \subset \{-1,+1\}^m$ such that $|\Sigma|-1\geq 2^{m/8}$, and for every distinct $\sigma,\sigma'\in \Sigma$, we have $d_{H}(\sigma,\sigma')\geq \frac{m}{8}$.

\noindent \textbf{KL divergence bound:} We now compute KL divergence between $P^{\sigma}$ and $P^{\sigma'}$. Because the label prior, namely the marginal distribution of \(Y\), is the same for all \(\sigma\), we have
\[
D_{\mathrm{KL}}(P^\sigma\|P^{\sigma'})
=
\frac12 D_{\mathrm{KL}}(P_0^\sigma\|P_0^{\sigma'})
+
\frac12 D_{\mathrm{KL}}(P_1^\sigma\|P_1^{\sigma'}).
\]

If $\sigma$ and $\sigma'$ differ in $k$ coordinates, then, for each coordinate in which they differ, the two conditional distributions interchange the two masses
$
\frac{\alpha_j(1+\gamma)}{m}$ and $\frac{\alpha_j}{m}$. Hence, each differing coordinate contributes
\[
\frac{\alpha_0(1+\gamma)}{m}
\log
\frac{\alpha_0(1+\gamma)/m}{\alpha_0/m}
+
\frac{\alpha_0}{m}
\log
\frac{\alpha_0/m}{\alpha_0(1+\gamma)/m}
\]
to $D_{\mathrm{KL}}\!\left(P_0^\sigma \,\middle\|\, P_0^{\sigma'}\right)$. Therefore,
\[
D_{\mathrm{KL}}\!\left(P_0^\sigma \,\middle\|\, P_0^{\sigma'}\right)
=
\frac{k\alpha_0}{m}
\left[
(1+\gamma)\log(1+\gamma)
+
\log \frac{1}{1+\gamma}
\right]\leq \frac{k\alpha_0\gamma^2}{m}\leq \alpha_0\gamma^2.
\]
Consequently,
\[
D_{\mathrm{KL}}\!\left(P^\sigma \,\middle\|\, P^{\sigma'}\right)
\le
\frac{1}{2}(\alpha_0+\alpha_1)\gamma^2.
\]
For n i.i.d. samples, we have 
\[
D_{\mathrm{KL}}\!\left((P^\sigma)^{\otimes n} \,\middle\|\, (P^{\sigma'})^{\otimes n}\right)
\le
\frac{1}{2}n(\alpha_0+\alpha_1)\gamma^2 .
\]
For $\gamma=\frac{1}{16}\sqrt{\frac{m}{n}}$, we have 
\begin{align}\label{kl_bound}
D_{\mathrm{KL}}\!\left((P^\sigma)^{\otimes n} \,\middle\|\, (P^{\sigma'})^{\otimes n}\right) \leq \eta \log{|\Sigma|}.
\end{align}
where $\eta=\frac{1}{32}$.

\noindent \textbf{Reduction to randomized Tsybakov's method (Proposition \ref{minimax_tsybakov}):}
The randomized learning rule
induces the randomized estimator $
    \widetilde\sigma
    =
    \widehat\sigma(\widetilde h_0,\widetilde h_1)
    =
    \widehat\sigma(S^n,\zeta)$ of the unknown parameter \(\sigma\). We apply Proposition~\ref{minimax_tsybakov} with parameter space
\(\Theta=\Sigma\), semi-metric $\operatorname{dist}(\sigma,\sigma')
    =
    d_H(\sigma,\sigma')$, and statistical model $
    \{\Pi^\sigma:\sigma\in\Sigma\}$, where $\Pi^\sigma= \left(P^\sigma\right)^{\otimes n}$ denotes the distribution of the sample $S^{n}$. For any distinct
\(\sigma,\sigma'\in\Sigma\), we have $
    d_H(\sigma,\sigma')\ge \frac{m}{8}$. Thus condition \((i)\) of Proposition~\ref{minimax_tsybakov} holds with $    s=\frac{m}{16}$. Moreover, by \eqref{kl_bound}, condition \((ii)\) of Proposition~\ref{minimax_tsybakov} holds with \(\eta=1/32\). Therefore, for $M+1=|\Sigma|$
with
\[
L_{M,\eta}
:=
\frac{\sqrt M}{1+\sqrt M}
\left(
1-2\eta-\sqrt{\frac{2\eta}{\log M}}
\right),
\]
Proposition~\ref{minimax_tsybakov} yields
\[
\inf_{\widetilde\sigma}
\sup_{\sigma\in\Sigma}
\Pi^\sigma
\left(
    \mathbb E_\zeta
    \left[
        d_H(\widetilde\sigma,\sigma)
    \right]
    \ge
    \frac{sL_{M,\eta}}{2}
\right)
\ge
\frac{L_{M,\eta}}{2-L_{M,\eta}}.
\]
Since \(s=m/16\), $\eta=1/32$, and $M\geq 2$, we can write
\[
\inf_{\widetilde\sigma}
\sup_{\sigma\in\Sigma}
\Pi^\sigma
\left(
    \mathbb E_\zeta
    \left[
        d_H(\widetilde\sigma,\sigma)
    \right]
    \ge
    \frac{m}{128}
\right)
\ge
\frac{1}{5}.
\]
The bound applies in particular to the randomized estimator induced by the
learning rule. Hence there exists \(\sigma\in\Sigma\) such that
\begin{align}
\Pi^\sigma
\left(
    \mathbb E_\zeta
    \left[
        d_H(\widetilde\sigma,\sigma)
    \right]
    \ge
    \frac{m}{128}
\right)
\ge
\frac{1}{5}.
\label{eq:rand-fano-event}
\end{align}

We have $\Pi^\sigma(\bar{\mathcal{E}}_{\sigma})\geq 1-\delta$. Then, we obtain that
\[
\Pi^\sigma\left(
\mathbb E_\zeta
    \left[
        d_H(\widetilde\sigma,\sigma)
    \right]\ge \frac{m}{128}
\ \text{and}\ 
\mathcal{E}_{\sigma}
\right)
\ge \frac{1}{5}-\delta \geq c.
\]
On this event, using \eqref{eq:rand-key-reduction-amb} 
we get \[
\begin{aligned}
\overline R_{\mathrm{amb}}^\sigma(\widetilde h_0,\widetilde h_1)
\ge
c_\alpha
\left[
    \frac{\gamma}{3m}
    \cdot
    \frac{m}{128}
    -
    \frac13
    \left(
        \frac{r_0}{\alpha_0}
        +
        \frac{r_1}{\alpha_1}
    \right)
\right]=c_\alpha
\left[
    \frac{1}{6144}\sqrt{\frac{m}{n}}
    -
    \frac13
    \left(
        \frac{r_0}{\alpha_0}
        +
        \frac{r_1}{\alpha_1}
    \right)
\right]\geq c_1\sqrt{\frac{d_{\mathcal{H}}}{n}}
\end{aligned}
\]
for some constant $c_1>0$ depending only on $\alpha_0,\alpha_1$, provided that $\frac{r_0}{\alpha_0}
    +
    \frac{r_1}{\alpha_1}
    \le c'\sqrt{\frac{d_{\mathcal{H}}}{n}}$ for a sufficiently small constant $c'$. Since the oracle pair has zero ambiguity, $R_{\mathrm{amb}}^{*,\sigma}=0$, we have $
    \overline{\mathcal E}_{\mathrm{amb}}
    (\widetilde h_0,\widetilde h_1)
    =
    \overline R_{\mathrm{amb}}^\sigma(\widetilde h_0,\widetilde h_1)$. Therefore, \[
\Pi^\sigma
\left(
    \overline{\mathcal E}_{\mathrm{amb}}
    (\widetilde h_0,\widetilde h_1)
    \ge
    c_1\sqrt{\frac{d_{\mathcal{H}}}{n}}
\right)
\ge
c.
\]
Consequently, \[
\sup_{\sigma\in\Sigma}
\Pi^\sigma
\left(
    \overline{\mathcal E}_{\mathrm{amb}}
    (\widetilde h_0,\widetilde h_1)
    \ge
    c_1\sqrt{\frac{d_{\mathcal{H}}}{n}}
\right)
\ge
c.
\]
This proves the minimax lower bound.

\noindent\textbf{Case (II): $\mathbf{4\leq d_{\mathcal H}\leq 17}$.}
In this case, we use a binary testing construction. Since $d_{\mathcal H}\geq 4$,
we can choose four points $
x_0,x_1,x_{1,1},x_{2,1}\in\mathcal X$ that are shattered by $\mathcal H$. We use the same construction as in
Case~(I), specialized to $m=1$. In particular, let
$\sigma\in\{-1,+1\}$ and define
\[
(a^\sigma,b^\sigma)
:=
\begin{cases}
(x_{1,1},x_{2,1}), & \text{if } \sigma=1,\\
(x_{2,1},x_{1,1}), & \text{if } \sigma=-1.
\end{cases}
\]
For each $\sigma\in\{-1,+1\}$, define $P^\sigma$ as in Case~(I), with $
\gamma=\frac{1}{8\sqrt{n}}$. The same argument as in Case~(I) shows that there exists a feasible oracle
pair $(h_{0,\sigma}^*,h_{1,\sigma}^*)$ such that $
R_{\mathrm{amb}}^{*,\sigma}=0$. Moreover, the ambiguity-to-Hamming reduction
\eqref{eq:rand-key-reduction-amb}, specialized to $m=1$, gives, on the
feasibility event $\overline{\mathcal E}_\sigma$,
\[
\overline R_{\mathrm{amb}}^\sigma
(\widetilde h_0,\widetilde h_1)
\geq
c_\alpha
\left[
\frac{\gamma}{3}
\mathbb{E}_\zeta
\left[
d_H(\widetilde\sigma,\sigma)
\right]
-
\frac{1}{3}
\left(
\frac{r_0}{\alpha_0}
+
\frac{r_1}{\alpha_1}
\right)
\right].
\]

It remains to lower bound the error in estimating the binary parameter
$\sigma$. Let
\[
\Pi^\sigma=(P^\sigma)^{\otimes n},
\qquad
\sigma\in\{-1,+1\}.
\]
By the same KL-divergence calculation as in Case~(I),
\[
D_{\mathrm{KL}}
\left(
\Pi^{+1}\Vert\Pi^{-1}
\right)
\leq
\frac{1}{2}n(\alpha_0+\alpha_1)\gamma^2
=
\frac{\alpha_0+\alpha_1}{128}.
\]
Since $\alpha_0,\alpha_1<49/100$, Pinsker's inequality implies
\[
\mathrm{TV}
\left(
\Pi^{+1},\Pi^{-1}
\right)
\leq
\sqrt{
\frac{1}{2}
D_{\mathrm{KL}}
\left(
\Pi^{+1}\Vert\Pi^{-1}
\right)
}
\leq
\frac{1}{16}.
\]
Therefore, by Le Cam's two-point method, for every randomized estimator
$\widetilde\sigma=\widetilde\sigma(S^n,\zeta)$,
\[
\max_{\sigma\in\{-1,+1\}}
\mathbb{E}_{\Pi^\sigma}
\mathbb{E}_\zeta
\left[
d_H(\widetilde\sigma,\sigma)
\right]
\geq
\frac{1}{2}
\left(
1-
\mathrm{TV}
\left(
\Pi^{+1},\Pi^{-1}
\right)
\right)
\geq
\frac{15}{32}.
\]
Hence, for every randomized learner, there
exists $\sigma\in\{-1,+1\}$ such that $
\mathbb{E}_{\Pi^\sigma}
\mathbb{E}_\zeta
\left[
d_H(\widetilde\sigma,\sigma)
\right]
\geq
\frac{15}{32}$. Define $
Q_\sigma(S^n)
:=
\mathbb{E}_\zeta
\left[
d_H(\widetilde\sigma,\sigma)
\right]$. Since $0\leq Q_\sigma(S^n)\leq 1$, we have
\[
\mathbb{E}_{\Pi^\sigma}[Q_\sigma]
\leq
\frac{1}{8}
+
\frac{7}{8}
\Pi^\sigma
\left(
Q_\sigma\geq\frac{1}{8}
\right).
\]
Therefore,
\[
\Pi^\sigma
\left(
\mathbb{E}_\zeta
\left[
d_H(\widetilde\sigma,\sigma)
\right]
\geq
\frac{1}{8}
\right)
\geq
\frac{11}{28}.
\]
Since $
\Pi^\sigma(\overline{\mathcal E}_\sigma)\geq 1-\delta$, we obtain
\[
\Pi^\sigma
\left(
\left\{
\mathbb{E}_\zeta
\left[
d_H(\widetilde\sigma,\sigma)
\right]
\geq
\frac{1}{8}
\right\}
\cap
\overline{\mathcal E}_\sigma
\right)
\geq
\frac{11}{28}-\delta
\geq c.
\]

On this event, the ambiguity-to-Hamming reduction gives
\[
\overline R_{\mathrm{amb}}^\sigma
(\widetilde h_0,\widetilde h_1)
\geq
c_\alpha
\left[
\frac{\gamma}{24}
-
\frac{1}{3}
\left(
\frac{r_0}{\alpha_0}
+
\frac{r_1}{\alpha_1}
\right)
\right].
\]
Since $\gamma=1/(8\sqrt{n})$, we have
\[
\overline R_{\mathrm{amb}}^\sigma
(\widetilde h_0,\widetilde h_1)
\geq
c_\alpha
\left[
\frac{1}{192\sqrt{n}}
-
\frac{1}{3}
\left(
\frac{r_0}{\alpha_0}
+
\frac{r_1}{\alpha_1}
\right)
\right].
\]
By assumption,
\[
\frac{r_0}{\alpha_0}
+
\frac{r_1}{\alpha_1}
\leq
c'\sqrt{\frac{d_{\mathcal H}}{n}}
\leq
c'\sqrt{\frac{17}{n}}.
\]
Thus, for a sufficiently small constant $c'>0$, we have $
\overline R_{\mathrm{amb}}^\sigma
(\widetilde h_0,\widetilde h_1)
\geq
\frac{c_1}{\sqrt{n}}$, for some constant $c_1>0$ depending only on $\alpha_0$ and $\alpha_1$.
Since $d_{\mathcal H}\leq 17$, we have $
\frac{1}{\sqrt{n}}
\geq
\frac{1}{\sqrt{17}}
\sqrt{\frac{d_{\mathcal H}}{n}}$. Consequently, for some constant $c_2>0$ depending only on
$\alpha_0$ and $\alpha_1$, we have $
\overline R_{\mathrm{amb}}^\sigma
(\widetilde h_0,\widetilde h_1)
\geq
c_2\sqrt{\frac{d_{\mathcal H}}{n}}$. Since $R_{\mathrm{amb}}^{*,\sigma}=0$, we have $
\overline{\mathcal E}_{\mathrm{amb}}
(\widetilde h_0,\widetilde h_1)
=
\overline R_{\mathrm{amb}}^\sigma
(\widetilde h_0,\widetilde h_1)$. Therefore,
\[
\Pi^\sigma
\left(
\overline{\mathcal E}_{\mathrm{amb}}
(\widetilde h_0,\widetilde h_1)
\geq
c_2\sqrt{\frac{d_{\mathcal H}}{n}}
\right)
\geq c.
\]
This proves the lower bound for the case $4\leq d_{\mathcal H}\leq 17$.

\section{Proof of Theorem \ref{thm:empirical_surrogate_upper_bound} [Surrogate upper bound]}
\begin{proof}
For $j\in\{0,1\}$, consider the composed class
\[
\phi\circ(-\mathcal F)
:=
\left\{
x\mapsto\phi(-f(x)):f\in\mathcal F
\right\}.
\]

By the contraction inequality
\cite[Theorem~12]{bartlett2002rademacher} and Assumption \ref{ass:rademacher_complexity}, for every $m\geq1$,
\begin{align}\label{contraction_bound}
\mathfrak R_m\bigl(\phi\circ(-\mathcal F)\bigr)
\leq
2L_\phi\mathfrak R_m(\mathcal F)
\leq
\frac{2L_\phi B_{\mathcal F}}{\sqrt m}.
\end{align}
Moreover, every function in $\phi\circ(-\mathcal F)$ takes values in
$[0,C_\phi]$. For each $j\in\{0,1\}$, define
\[
Z_j
:=
\sup_{f\in\mathcal F}
\left|
R_{\phi,j}(f)-\widehat R_{\phi,j}(f)
\right|.
\]
For every integer $m\geq1$, conditional on the event $\{n_j=m\}$, the
observations $\{X_i:Y_i=j\}$ are i.i.d.\ from $P_j$. Therefore, the
Rademacher uniform deviation inequality \cite[Theorem~3.3]{mohri2018foundations} implies that
\[
\mathbb P\left(
Z_j>
2\mathfrak R_m\bigl(\phi\circ(-\mathcal F)\bigr)
+
C_\phi\sqrt{\frac{\log(6/\delta)}{2m}}
\,\middle|\,
n_j=m
\right)
\leq
\frac{\delta}{3}.
\]
Using \eqref{contraction_bound}, we obtain
\[
\mathbb P\left(
Z_j>
\frac{4L_\phi B_{\mathcal F}}{\sqrt m}
+
C_\phi\sqrt{\frac{\log(6/\delta)}{2m}}
\,\middle|\,
n_j=m
\right)=\mathbb P\left(
Z_j > \epsilon_\phi(m,\delta)
\,\middle|\,
n_j=m
\right)
\leq
\frac{\delta}{3}.
\]
Because $\mathcal E_n$ is determined by the class counts and implies
$n_j\geq1$, the law of total probability yields

Then,
\[
\begin{aligned}
\mathbb P\left(
Z_j>\epsilon_\phi(n_j,\delta)\,\middle|\,
\mathcal E_n
\right)
=
\sum_{m=1}^{n-1}
\mathbb P\left(
Z_j>\epsilon_\phi(m,\delta)
\,\middle|\,
n_j=m
\right)
\mathbb P\left(
n_j=m \,\middle|\,\mathcal E_n
\right)\leq\frac{\delta}{3}
\sum_{m=1}^{n-1}
\mathbb P\left(
n_j=m\,\middle|\,\mathcal E_n
\right)
\leq
\frac{\delta}{3}.
\end{aligned}
\] 

Thus, Defining the event $
\mathcal A
:=\mathcal E_n\cap
\left\{
Z_j\leq\epsilon_\phi(n_j,\delta) \ \text{for} \ j=0,1\right\}$,
we have $
\mathbb P\left(\mathcal A\,\middle|\,\mathcal E_n\right)
\geq
1-\frac{2\delta}{3}$.

We next control the surrogate ambiguity objective. Define
\[
\mathcal G_\diamond
:=
\left\{
x\mapsto
\ell_{\mathrm{amb},\phi}^{\diamond}(f_0,f_1;x):
(f_0,f_1)\in\mathcal F\times\mathcal F
\right\}.
\]
For the additive surrogate ambiguity loss, the fact that the positive-part
map is $1$-Lipschitz gives
\[
\left|
[u+v-1]_+-[u'+v'-1]_+
\right|
\leq
|u-u'|+|v-v'|.
\]
Consequently,
\[
\begin{aligned}
\left|
\ell_{\mathrm{amb},\phi}^{\mathrm{add}}
(f_0,f_1;x)
-
\ell_{\mathrm{amb},\phi}^{\mathrm{add}}
(g_0,g_1;x)
\right|&\leq
\left|
\phi(f_0(x))-\phi(g_0(x))
\right|
+
\left|
\phi(f_1(x))-\phi(g_1(x))
\right|\\
&\leq
L_\phi
\left(
|f_0(x)-g_0(x)|
+
|f_1(x)-g_1(x)|\right).
\end{aligned}
\]
Moreover, $
0
\leq
\ell_{\mathrm{amb},\phi}^{\mathrm{add}}(f_0,f_1;x)
\leq
2C_\phi$. For the multiplicative surrogate ambiguity loss, we have
\[
\begin{aligned}
\left|
\phi(f_0(x))\phi(f_1(x))
-
\phi(g_0(x))\phi(g_1(x))
\right|&\leq
\phi(f_1(x))
\left|
\phi(f_0(x))-\phi(g_0(x))
\right|
+
\phi(g_0(x))
\left|
\phi(f_1(x))-\phi(g_1(x))
\right|\\
&\leq
C_\phi L_\phi
\left(
|f_0(x)-g_0(x)|+
|f_1(x)-g_1(x)|
\right),
\end{aligned}
\]
Also, $
0
\leq
\ell_{\mathrm{amb},\phi}^{\mathrm{prod}}(f_0,f_1;x)
\leq
C_\phi^2$. Recalling that $
\kappa_{\mathrm{add}}=1$, and $\kappa_{\mathrm{prod}}=C_\phi$, the vector contraction inequality \citep[Corollary~1 and Section~4.1]{maurer2016vector} and
Assumption~\ref{ass:rademacher_complexity} imply that there exists a
universal constant $C>0$ such that
\[
\mathfrak R_n(\mathcal G_\diamond)
\leq
C\kappa_\diamond L_\phi
\left(
\mathfrak R_n(\mathcal F)
+
\mathfrak R_n(\mathcal F)
\right)
\leq
C\frac{\kappa_\diamond L_\phi B_{\mathcal F}}{\sqrt n},
\]
Define
\[
Z_{\mathrm{amb}}^\diamond
:=
\sup_{(f_0,f_1)\in\mathcal F\times\mathcal F}
\left|
R_{\mathrm{amb},\phi}^{\diamond}(f_0,f_1)
-
\widehat R_{\mathrm{amb},\phi}^{\diamond}(f_0,f_1)
\right|.
\]
Since every function in $\mathcal G_\diamond$ takes values in
$[0,M_\diamond]$, the Rademacher uniform deviation inequality gives
\[
\mathbb P\left(
Z_{\mathrm{amb}}^\diamond
>
2\mathfrak R_n(\mathcal G_\diamond)
+
M_\diamond
\sqrt{\frac{\log(6/\delta)}{2n}}
\right)=\mathbb P\left(
Z_{\mathrm{amb}}^\diamond
>\epsilon_{\mathrm{amb},\phi}^\diamond(n,\delta)
\right)
\leq
\frac{\delta}{3}.
\]
Let $
\mathcal A_{\mathrm{amb}}^\diamond
:=
\left\{
Z_{\mathrm{amb}}^\diamond
\leq
\epsilon_{\mathrm{amb},\phi}^\diamond(n,\delta)
\right\}$. Then we obtain
\[
\begin{aligned}
\mathbb P\left(
\mathcal A\cap\mathcal A_{\mathrm{amb}}^\diamond
\,\middle|\,\mathcal E_n
\right)
&\ge 1-\frac{2\delta}{3}-\frac{\delta}{3\mathbb P(\mathcal E_n)}
\end{aligned}
\]
It remains to establish the claimed conclusions, conditional on
$\mathcal E_n$, on the event
$\mathcal A\cap\mathcal A_{\mathrm{amb}}^\diamond$. On $\mathcal A$, for $j\in\{0,1\}$ we have
\[
\widehat R_{\phi,j}(f_j^{\diamond,*})
\leq
R_{\phi,j}(f_j^{\diamond,*})
+
\epsilon_\phi(n_j,\delta)
\leq
\alpha_j+\epsilon_\phi(n_j,\delta).
\]
Thus, $f^{\diamond,*}=(f_0^{\diamond,*},f_1^{\diamond,*})$ is feasible for the empirical optimization
problem. By the empirical optimality of
$\widehat f^\diamond
=(\widehat f_0^\diamond,\widehat f_1^\diamond)$, we have
$
\widehat R_{\mathrm{amb},\phi}^{\diamond}
(\widehat f_0^\diamond,\widehat f_1^\diamond)
\leq
\widehat R_{\mathrm{amb},\phi}^{\diamond}
(f_0^{\diamond,*},f_1^{\diamond,*})$. On $\mathcal A_{\mathrm{amb}}^\diamond$, it follows that\[
\begin{aligned}
R_{\mathrm{amb},\phi}^{\diamond}
(\widehat f_0^\diamond,\widehat f_1^\diamond)
\leq
\widehat R_{\mathrm{amb},\phi}^{\diamond}
(\widehat f_0^\diamond,\widehat f_1^\diamond)
+
\epsilon_{\mathrm{amb},\phi}^{\diamond}(n,\delta)&\leq
\widehat R_{\mathrm{amb},\phi}^{\diamond}
(f_0^{\diamond,*},f_1^{\diamond,*})
+
\epsilon_{\mathrm{amb},\phi}^{\diamond}(n,\delta)\\
&\leq
R_{\mathrm{amb},\phi}^{\diamond}
(f_0^{\diamond,*},f_1^{\diamond,*})
+
2\epsilon_{\mathrm{amb},\phi}^{\diamond}(n,\delta).
\end{aligned}
\]
Therefore, $
\mathcal E_{\mathrm{amb},\phi}^{\diamond}
(\widehat f_0^\diamond,\widehat f_1^\diamond)
\leq
2\epsilon_{\mathrm{amb},\phi}^{\diamond}(n,\delta).$ Finally, empirical feasibility gives
\[
\widehat R_{\phi,j}(\widehat f_j^\diamond)
\leq
\alpha_j+\epsilon_\phi(n_j,\delta),
\qquad j\in\{0,1\}.
\]
Hence, on $\mathcal A$,
\[
\begin{aligned}
R_{\phi,j}(\widehat f_j^\diamond)
\leq
\widehat R_{\phi,j}(\widehat f_j^\diamond)
+
\epsilon_\phi(n_j,\delta)\leq
\alpha_j+2\epsilon_\phi(n_j,\delta).
\end{aligned}
\]
Since the surrogate loss upper bounds the corresponding indicator loss, $
R_j(h_{\widehat f_j^\diamond})
\leq
R_{\phi,j}(\widehat f_j^\diamond)$, and therefore
\[
R_j(h_{\widehat f_j^\diamond})
\leq
\alpha_j+2\epsilon_\phi(n_j,\delta),
\qquad j\in\{0,1\}.
\]
This completes the proof.
\end{proof}

\section{Proof of Theorem~\ref{thm:optimization}}

\begin{proof}
According to~\cite{NIPS2012_c52f1bd6},
Algorithm~\ref{algorithm: SGD with One Projection} guarantees that,
conditional on the training sample, with probability at least
$1-\delta$, we have
\[
\widehat R_{\mathrm{amb},\phi}^{\mathrm{add}}
(f_{\widehat\theta_0},f_{\widehat\theta_1})
-
\widehat R_{\mathrm{amb},\phi}^{\mathrm{add}}
(f_{\widehat\theta_0^\star},f_{\widehat\theta_1^\star})
\lesssim
\frac{1}{\sqrt T}
\left[
\left(1+\frac{G_1^2}{\rho^2}\right)c_\delta
+\frac{G_1^2G_2^2}{\rho^2c_\delta}
+G_1\sqrt{\ln\frac{2}{\delta}}
\right].
\]
Taking
\[
T\gtrsim
\frac{1}{
\bigl(\epsilon_{\mathrm{amb},\phi}^{\mathrm{add}}(n,\delta)\bigr)^2}
\left[
\left(1+\frac{G_1^2}{\rho^2}\right)c_\delta
+\frac{G_1^2G_2^2}{\rho^2c_\delta}
+G_1\sqrt{\ln\frac{2}{\delta}}
\right]^2
\]
with a sufficiently large implicit constant ensures that
\[
\widehat R_{\mathrm{amb},\phi}^{\mathrm{add}}
(f_{\widehat\theta_0},f_{\widehat\theta_1})
-
\widehat R_{\mathrm{amb},\phi}^{\mathrm{add}}
(f_{\widehat\theta_0^\star},f_{\widehat\theta_1^\star})
\le
\epsilon_{\mathrm{amb},\phi}^{\mathrm{add}}(n,\delta).
\]
Following the same uniform-deviation argument as in the proof of
Theorem \ref{thm:empirical_surrogate_upper_bound}, this additional optimization error yields
\[
\mathcal E_{\mathrm{amb},\phi}^{\mathrm{add}}
(f_{\widehat\theta_0},f_{\widehat\theta_1})
\le
3\epsilon_{\mathrm{amb},\phi}^{\mathrm{add}}(n,\delta).
\]
Since $\widehat\theta\in K$, the empirical constraints hold, so the
feasibility guarantees follow exactly as in Theorem \ref{thm:empirical_surrogate_upper_bound}.
A union bound accounts for the additional failure probability of
at most $\delta$.
\end{proof}

\section{Auxiliary Results}
\begin{proof}[Proof of Proposition \ref{minimax_tsybakov}]\label{proof_Auxiliary_Results}
Fix an arbitrary randomized estimator $\widetilde{\theta}=\widetilde{\theta}(S^n,\zeta)$. Let $\nu$ denote the distribution of $\zeta$. Under parameter $\theta$, the joint law of $(S^n,\zeta)$ is $P_{\theta}^{\otimes n}\otimes \nu$. For every $j=1,...,M$,
\[
D_{\mathrm{KL}}\!\left(
P_{\theta_j}^{\otimes n}\otimes \nu
\,\middle\|\,
P_{\theta_0}^{\otimes n}\otimes \nu
\right)
=
D_{\mathrm{KL}}\!\left(
P_{\theta_j}^{\otimes n}
\,\middle\|\,
P_{\theta_0}^{\otimes n}
\right).
\]
Therefore,
\[
\frac{1}{M}\sum_{j=1}^{M}
D_{\mathrm{KL}}\!\left(
P_{\theta_j}^{\otimes n}\otimes \nu
\,\middle\|\,
P_{\theta_0}^{\otimes n}\otimes \nu
\right)
\leq \eta \log M .
\]
Now view $\widetilde{\theta}(S^n,\zeta)$ as an estimator based on the enlarged observation $(S^n,\zeta)$. Applying the usual Fano/Tsybakov lower bound [\cite{Tsybakov:1315296}, Theorem 2.5] to the family of distributions $\left\{
P_{\theta_j}^{\otimes n}\otimes \nu : j=0,1,\ldots,M
\right\}$, we obtain
\[
\sup_{0\leq j\leq M}
\left(P_{\theta}^{\otimes n}\otimes \nu\right)
\left(
d\bigl(\widetilde\theta(S^n,\zeta),\theta\bigr)\ge s
\right)
\ge L_{M,\eta}.
\]

Equivalently,
\[
\sup_{0\le j\le M}
\mathbb{E}_{S^n\sim P_{\theta_j}^{\otimes n}}
\mathbb{E}_{\zeta}
\left[
\mathbf{1}\left\{
d\bigl(\widetilde{\theta}(S^n,\zeta),\theta_j\bigr)\ge s
\right\}
\right]
\ge L_{M,\eta}.
\]
Hence there exists some $j^\star\in\{0,\ldots,M\}$ such that
\[
\mathbb{E}_{\theta_{j^\star}}\mathbb{E}_{\zeta}
\left[
\mathbf{1}\left\{
d\bigl(\widetilde{\theta}(S^n,\zeta),\theta_{j^\star}\bigr)\ge s
\right\}
\right]
\ge L_{M,\eta}.
\]
Define
\[
Q(S^n)
:=
\mathbb{E}_{\zeta}
\left[
\mathbf{1}\left\{
d\bigl(\widetilde{\theta}(S^n,\zeta),\theta_{j^\star}\bigr)\ge s
\right\}
\right].
\]
For any $t\in(0,L_{M,\eta})$, we have
\[
L_{M,\eta}\leq \mathbb{E}_{\theta_{j^\star}}\!\left[Q(S^n)\right]
\le
t
+
(1-t)
P_{\theta_{j^\star}}^{\otimes n}
\left(Q(S^n)\ge t\right).
\]
Therefore, $P_{\theta_{j^*}}^{\otimes n}\left(Q(S^n)\ge t\right)
\ge
\frac{L_{M,\eta}-t}{1-t}.$ For $t=\frac{L_{M,\eta}}{2}$, we get \[P_{\theta_{j^*}}^{\otimes n}
\left(
Q(S^n)\ge \frac{L_{M,\eta}}{2}
\right)
\ge
\frac{L_{M,\eta}}{2-L_{M,\eta}}.\]
Now observe that, given $S^n$,
\[
\mathbb{E}_{\zeta}
\left[
d\bigl(\widetilde{\theta}(S^n,\zeta),\theta_{j^*}\bigr)
\right]
\ge
s Q(S^n).
\]
Hence,
\[
\left\{
Q(S^n)\ge \frac{L_{M,\eta}}{2}
\right\}
\subseteq
\left\{
\mathbb{E}_{\zeta}
\left[
d\bigl(\widetilde{\theta}(S^n,\zeta),\theta_{j^*}\bigr)
\right]
\ge
\frac{sL_{M,\eta}}{2}
\right\}.
\]
Therefore,
\[
P_{\theta_{j^*}}^{\otimes n}
\left(
\mathbb{E}_{\zeta}
\left[
d\bigl(\widetilde{\theta}(S^n,\zeta),\theta_{j^*}\bigr)
\right]
\ge
\frac{sL_{M,\eta}}{2}
\right)
\ge
\frac{L_{M,\eta}}{2-L_{M,\eta}}.
\]
This holds for every randomized estimator $\widetilde{\theta}$. Taking the infimum over all randomized estimators gives
\[
\inf_{\widetilde{\theta}}
\sup_{\theta\in\Theta}
P_{\theta}^{\otimes n}
\left(
\mathbb{E}_{\zeta}
\left[
d\bigl(\widetilde{\theta}(S^n,\zeta),\theta\bigr)
\right]
\ge
\frac{sL_{M,\eta}}{2}
\right)
\ge
\frac{L_{M,\eta}}{2-L_{M,\eta}}.
\]
This concludes the proof.
\end{proof}

\begin{lemma}\label{lemma:VC_bound} Let $\mathcal H$ be a hypothesis class with $0<d_{\mathcal{H}}<\infty$, and define \[ \mathcal G := \left\{ g_{h_0,h_1}:x\mapsto h_0(x)h_1(x) \;:\; h_0,h_1\in\mathcal H \right\}. \] Then $ d_{\mathcal{G}}\leq C d_{\mathcal{H}}$ for a universal constant $C>0$. \end{lemma}

\begin{proof} Let $\Pi_{\mathcal H}(m)$ and $\Pi_{\mathcal G}(m)$ denote the growth functions of $\mathcal H$ and $\mathcal G$, respectively. For any fixed points $x_1,\ldots,x_m$, each labeling induced by a function $g_{h_0,h_1}\in\mathcal G$ is completely determined by the two labelings \[ \bigl(h_0(x_1),\ldots,h_0(x_m)\bigr) \qquad\text{and}\qquad \bigl(h_1(x_1),\ldots,h_1(x_m)\bigr). \] Therefore, $ \Pi_{\mathcal G}(m) \leq \Pi_{\mathcal H}(m)^2$. By Sauer's lemma, for $m\geq d_{\mathcal H}$, we have $\Pi_{\mathcal H}(m) \leq \left(\frac{em}{d_{\mathcal H}}\right)^{d_{\mathcal H}}$, and hence $ \Pi_{\mathcal G}(m) \leq \left(\frac{em}{d_{\mathcal H}}\right)^{2d_{\mathcal H}}$. Now suppose that $\mathcal G$ shatters $m$ points. If $m<d_{\mathcal H}$, then trivially $m=O(d_{\mathcal H})$. Otherwise, $m\geq d_{\mathcal H}$, and since $\Pi_{\mathcal G}(m)=2^m$, we obtain $2^m \leq \left(\frac{em}{d_{\mathcal H}}\right)^{2d_{\mathcal H}}$. Taking $m=10d_{\mathcal H}$, we obtain $ 2^{10d_{\mathcal H}} \leq (10e)^{2d_{\mathcal H}}$, which is impossible. Thus $\mathcal G$ cannot shatter $10d_{\mathcal H}$ points, and consequently $ d_{\mathcal G}<10d_{\mathcal H}$. This proves $d_{\mathcal G}\leq Cd_{\mathcal H}$ for a universal constant $C$.
\end{proof}

\end{document}